\documentclass{article}

\PassOptionsToPackage{numbers, compress}{natbib}

\usepackage[main, final]{neurips_2026}

\usepackage[utf8]{inputenc}   
\usepackage[T1]{fontenc}      
\usepackage{hyperref}         
\usepackage{url}              
\usepackage{booktabs}         
\usepackage{amsmath,amssymb,amsfonts} 
\usepackage{nicefrac}         
\usepackage{microtype}        
\usepackage{xcolor}           
\usepackage{graphicx}         
\usepackage{algorithm}        
\usepackage{algpseudocode}    
\usepackage{float}
\usepackage[most]{tcolorbox}
\usepackage{array}
\usepackage{siunitx}
\usepackage{wrapfig}
\usepackage{multirow}
\usepackage{amsthm}
\newtheorem{theorem}{Theorem}
\newtheorem{corollary}{Corollary}

\title{Conflict-Aware Logit Adapters for Utility-Preserving Anti-Distillation}

\author{%
  \textbf{Quoc Phong Dao\textsuperscript{1}}\footnotemark[1],
  \textbf{Binh Anh Nguyen\textsuperscript{1,3}}\footnotemark[1],
  \textbf{Linh Ngo Van\textsuperscript{1}}\footnotemark[2],\\
  \textbf{Thanh-Toan Do\textsuperscript{2}},
  \textbf{Mehrtash Harandi\textsuperscript{2}},
  \textbf{Trung Le\textsuperscript{2}}\footnotemark[2],
  \bigskip \\
\textsuperscript{1}Hanoi University of Science and Technology, 
\textsuperscript{2}Monash University,\\
\textsuperscript{3}Center for AI Research (CAIR), VinUniversity
}

\begin{document}
\renewcommand{\thefootnote}{\fnsymbol{footnote}}
\maketitle

\footnotetext[1]{Equal contribution}
\footnotetext[2]{Corresponding authors: Linh Ngo Van (\href{mailto:linhnv@soict.hust.edu.vn}{linhnv@soict.hust.edu.vn}) and Trung Le (\href{mailto:trunglm@monash.edu}{trunglm@monash.edu})}
\renewcommand*{\thefootnote}{\arabic{footnote}}

\begin{abstract}

While reasoning traces improve model capability, they also create a pathway for extracting supervision signals for downstream distillation. Initial efforts to mitigate this rely on decoding-time interventions, but they incur high inference overhead and often harm the teacher model's natural fluency. 
We propose Conflict-Aware Logit Adapters (CALA), a lightweight training-time defense. CALA attaches a small residual adapter to a frozen teacher model and trains only this adapter. A proxy student identifies positions where imitation is most sensitive, and the adapter perturbs the output distribution there within high-probability regions. A conflict-aware optimization balances competing anti-distillation and utility objectives, with regularization preserving the teacher’s performance and fluency.
At inference, CALA incurs modest overhead and requires no auxiliary model. Experiments on reasoning benchmarks show that CALA substantially reduces successful student imitation while maintaining near-original teacher performance and generation quality, offering a practical and efficient alternative to decoding-time approaches.
\end{abstract}

\section{Introduction}
\label{sec:introduction}
Large language models increasingly produce reasoning traces alongside final answers, as chain-of-thought prompting, reasoning-path sampling, and process supervision improve multi-step reasoning~\citep{wei2022chain,wang2023selfconsistency,lightman2024lets}. These traces also reveal intermediate decompositions and solution patterns that can serve as high-information training signal for smaller student models, with evidence spanning rationale-based distillation and broader instruction-tuning or imitation pipelines built from synthetic frontier-model outputs~\citep{hsieh2023distilling,yu2024metamath,wang2023selfconsistency,gudibande2023false}. This creates a security tension: reasoning-rich outputs improve usability, but may also facilitate model replication.

Although prior work has shown that black-box language models can leak substantial information through their outputs, including memorized data and signals useful for model extraction~\citep{carlini2021extracting,carlini2024stealing}, existing defenses are not well aligned with our setting. Output-privacy work primarily focuses on preventing memorization or training-data extraction, while data-poisoning approaches assume control over the learner’s training set~\citep{huang2020metapoison, geiping2020witches}. These assumptions do not match our deployment scenario, where the teacher is fixed, exposed through black-box access, and must continue producing useful outputs for end users.

To study this gap concretely, we focus on a specific threat model: \textit{an attacker collects reasoning traces emitted by a deployed teacher model and uses them as supervision to train a smaller student via standard next-token prediction~\citep{gudibande2023false}}. The attacker has black-box access to the teacher and observes full input-trace pairs, but has no access to model parameters or internal states~\citep{carlini2024stealing}. The defender, in contrast, controls only the distribution of outputs produced by the teacher, and must preserve their utility for end users while limiting their effectiveness as training data for downstream imitation. Under this setting, the closest prior work is Antidistillation Sampling (ADS)~\citep{savani2025antidistillation}, which perturbs the teacher's next-token distribution during decoding to reduce the value of released traces for downstream distillation. However, because ADS applies protection online, it requires proxy-model computation during generation and can introduce decoding overhead or surface artifacts. These limitations motivate a training-time alternative that replaces online perturbations with learned protective behavior in the teacher.

We propose \textbf{Conflict-Aware Logit Adapters (CALA)}, a training-time approach for limiting trace-based model imitation. At a high level, we attach a lightweight residual logit adapter to a frozen teacher and train only this module, amortizing protection into the model instead of imposing it through an online decoding rule. At deployment, generation follows the standard decoding pipeline, and the learned residual can be scaled by a single strength parameter without running an ADS-style proxy model at each token. Concretely, during training, a proxy student takes a single, virtual, non-committed gradient-ascent step on its next-token loss to expose positions where imitation is locally fragile. The adapter is then optimized to shift probability mass on restricted high-probability supports so that released traces become less useful for distillation, while utility-preserving regularizers and a conflict-aware update keep the distribution close to the original teacher and reduce surface artifacts.

\textbf{Contributions.} We make three contributions.
\textit{First}, we formulate Conflict-Aware Logit Adapters (CALA), a training-time defense for reasoning traces that targets the same broad threat as ADS while shifting the intervention from decoding to model training. This avoids online proxy-model computation at inference time and remains compatible with standard decoding choices.
\textit{Second}, we develop a practical proxy-guided training procedure in which a proxy student takes a virtual, non-committed ascent step to provide an anti-distillation signal, and a residual logit adapter is updated with a conflict-aware rule that balances anti-distillation pressure against utility preservation.
\textit{Third}, we empirically evaluate CALA by jointly measuring teacher utility, downstream student imitation, and surface fluency across reasoning benchmarks and evaluation-time operating points.

\section{Related Work}
\label{sec:related-work}

\textbf{Reasoning traces and distillation.}
Intermediate reasoning traces are now common in language-model interfaces. Chain-of-thought prompting, zero-shot reasoning prompts, self-consistency, and process supervision show that intermediate reasoning can improve performance on difficult multi-step tasks~\citep{wei2022chain,kojima2022zeroshot,wang2023selfconsistency,lightman2024lets}. The same traces can also serve as supervision: smaller models can inherit nontrivial reasoning behavior from teacher-generated rationales, step-by-step solutions, and synthetic instruction-following data~\citep{hsieh2023distilling,shridhar2023distilling,yu2024metamath,gu2024minillm}. Our work is motivated by this dual use: we aim to preserve useful reasoning traces while reducing their effectiveness as supervision for downstream imitation.

A complementary line improves transfer across tokenizers or architectures via cross-chain-of-thought and optimal-transport alignment, multi-cost Wasserstein objectives, dual-space sequence alignment, span-level representation alignment, and multi-granular layer trajectories~\citep{le2026discord,vuong2026mcwkd,vu2026dwakd,dao2026sra,chi2026mta}. Related approaches use teacher-guided token weighting for self-play and distillation~\citep{le2025swift} or transfer preferences across tokenizers~\citep{nguyen2026ctpd}. While these methods strengthen distillation rather than defend against it, they motivate evaluating defenses against diverse distillation strategies.

\textbf{Output leakage and defensive perspectives.}
A broader security literature has shown that model outputs can reveal valuable information about the underlying model or its training data. Prior work studies extraction of memorized training examples and reconstruction of model behavior from black-box access~\citep{carlini2021extracting,nasr2025scalable,carlini2024stealing}. Data-poisoning work studies a related but different setting, where carefully chosen upstream data induces predictable downstream behavior~\citep{huang2020metapoison,geiping2020witches,goldblum2022dataset}; similar risks have also been studied for preference data~\citep{rando2024universal}. In our setting, the defender does not control the attacker's training pipeline or dataset construction. Instead, the defender controls only the traces emitted by a deployed teacher and must preserve their utility for legitimate users.

\textbf{Protecting against distillation.}
Recent work has begun to study defenses that reduce the value of model outputs as supervision for unauthorized imitation. The most closely related work is Antidistillation Sampling (ADS)~\citep{savani2025antidistillation}, which targets the same utility-distillability tension for released reasoning traces. DOGe~\citep{li2025doge} is a training-time alternative that fine-tunes the teacher's output layer using an adversarial objective. More broadly, this line connects to work on making teachers harder to imitate~\citep{ma2021undistillable}, trace-rewriting defenses for unauthorized distillation~\citep{ma2026protecting}, watermarking methods for marking or protecting generated text~\citep{kirchenbauer2023watermark,gu2023learnability,sander2024watermarking,zhao2023protecting}, and decoding-time control methods that steer generation with auxiliary objectives~\citep{mudgal2023controlled,li2023contrastive,ji2023language,qin2022cold}. Our goal differs from watermarking-based attribution or detection: we aim to reduce the usefulness of released reasoning traces as distillation supervision.

\section{Motivation and Problem Formulation}
\label{sec:background_formulation}

\textbf{Trace-based distillation.}
We consider a deployed reasoning model that returns full reasoning traces to users. These traces are useful because they expose intermediate steps, but the same structure also makes them high-quality supervision for black-box imitation. Let $T_\Theta$ be a frozen teacher language model, and let $p_t=p(\cdot\mid x_{\leq t};\Theta)$ denote its next-token distribution over vocabulary $\mathcal{V}$. An adversary queries the teacher, collects input-trace pairs, and trains a smaller student $S_\phi$ with the standard next-token objective over the generated tokens,
\begin{equation}
\mathcal{L}_{\mathrm{CE}}(\phi)
= -\frac{1}{N}\sum_{t=1}^{N}\log p_\phi(x_{t+1}\mid x_{\leq t}),
\label{eq:student_ce}
\end{equation}
where $N$ is the number of next-token prediction terms in the trace. The attacker observes only text, not teacher parameters, hidden states, or logits. The defender therefore controls only the distribution of traces emitted by the teacher.

\begin{wrapfigure}{r}{0.52\textwidth}
    \vspace{-7pt}
    \centering
    \includegraphics[width=0.5\textwidth]{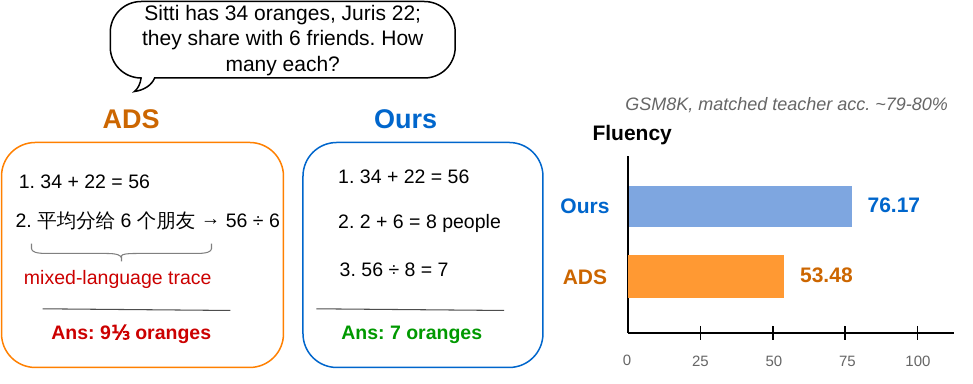}
    \caption{
    Representative GSM8K example: ADS can introduce visible surface artifacts such as mixed-language reasoning and malformed fragments, while CALA produces a cleaner trace.
    }
    \vspace{-8pt}
    \label{fig:intro_teaser}
\end{wrapfigure}
\textbf{Limitations of existing work.}
Antidistillation Sampling (ADS)~\citep{savani2025antidistillation} is a pioneering decoding-time defense for protecting reasoning traces against unauthorized distillation. It perturbs the teacher’s next-token distribution during generation by adding a proxy-student-derived anti-distillation penalty to the teacher logits, encouraging tokens that remain plausible under the original teacher while reducing the usefulness of the resulting traces for student training. However, our empirical evaluation reveals several practical limitations: ADS also incurs \textit{additional inference-time cost} because it requires proxy-model computation during generation. Moreover, our empirical evaluation shows that ADS can introduce surface-level artifacts in reasoning traces, such as switching to another language despite an English question, as shown in Figure~\ref{fig:intro_teaser}. We also observe cases where the final answer remains correct while the chain-of-thought becomes incoherent, containing meaningless tokens, unnatural fragments, or mixed-language reasoning. We refer to this behavior as \textbf{anti-distillation hacking}: the model preserves task accuracy by producing the correct final answer, but degrades the intermediate reasoning in ways that make the trace harder to distill (\textit{more examples can be found in Appendix \ref{app:samples}}). This suggests that evaluating only teacher accuracy is insufficient; a reliable defense must also assess the fluency and usability of the generated reasoning traces.

\textbf{Design goal.}
A defense against trace distillation should not simply degrade generation quality. If the teacher produces unnatural or low-utility traces, imitation becomes harder, but the model also becomes less useful to legitimate users. We seek a defended teacher that preserves fluent reasoning while reducing the value of its traces as distillation data. CALA moves the intervention to training time: we learn a lightweight residual adapter on top of a frozen teacher; at inference, its effect can be controlled by a single strength parameter without running a proxy model during decoding.

\section{Method}
\label{sec:method}
CALA aims to reduce the usefulness of released reasoning traces for next-token imitation while preserving utility and surface naturalness. As illustrated in Figure~\ref{fig:method_overview}, it trains a residual adapter through three components: a temporary proxy-student ascent step that exposes imitation-sensitive positions, support-restricted losses on high-probability tokens, and utility regularizers that anchor the adapted distribution to teacher-derived targets. A conflict-aware update then balances anti-distillation and utility gradients so that utility preservation does not erase the protective signal.

\begin{figure}[t]
    \centering
    \includegraphics[width=0.90\linewidth]{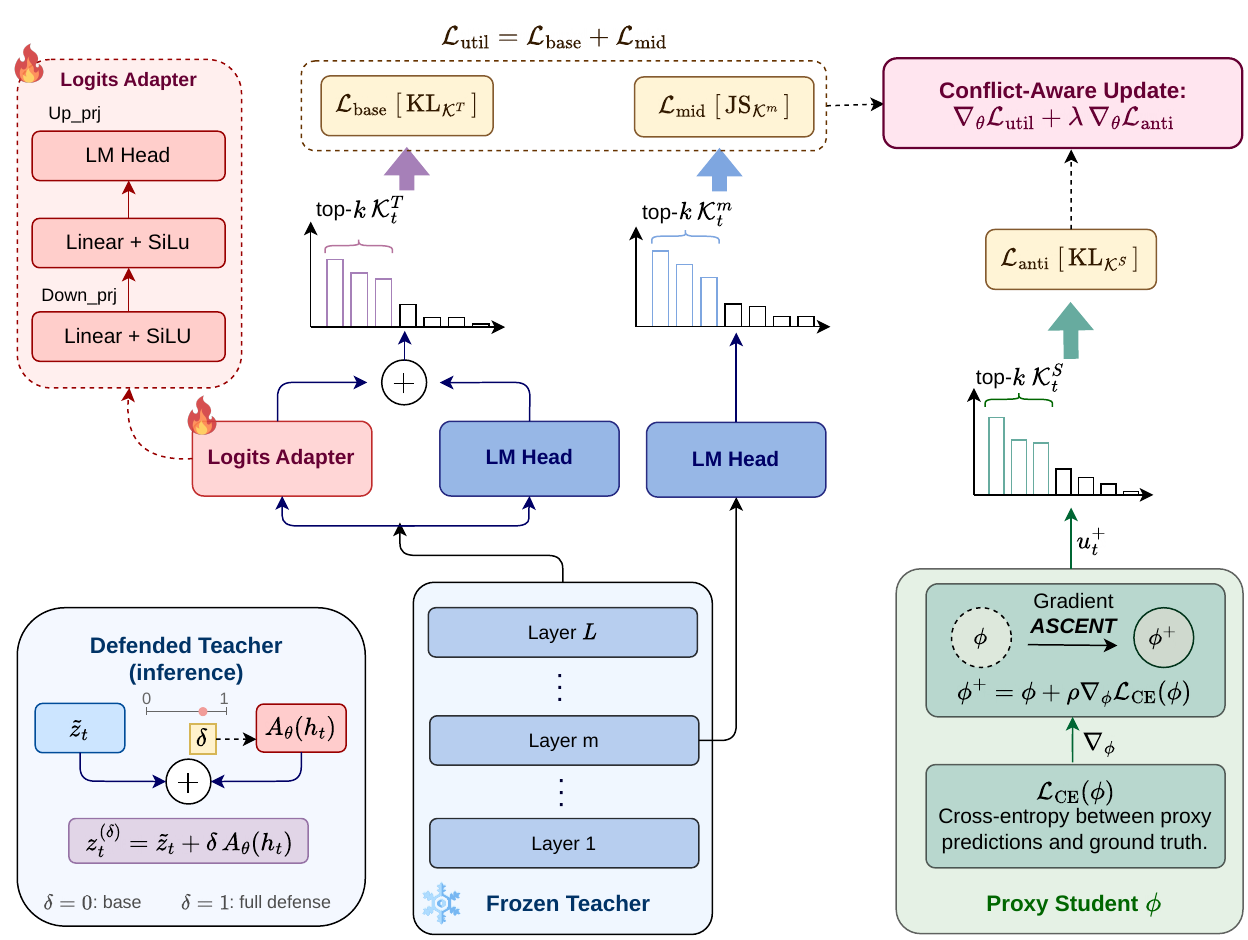}
    \caption{
    Overview of Conflict-Aware Logit Adapters (CALA). A trainable residual adapter modifies the frozen teacher's logits $z_t$, while proxy-guided anti-distillation and teacher-derived utility anchors jointly shape the adapted distribution.}
    \label{fig:method_overview}
\end{figure}

\subsection{Training-time defended teacher}

\textbf{Residual adapter.} We implement the defense as a small residual modification to a frozen teacher,
rather than fine-tuning the teacher itself. Let $h_t$ denote the final hidden state of the
frozen teacher at position $t$, and let $\tilde z_t = W_{\mathrm{lm}} h_t$ be the
logits produced by the original language-modeling head. We introduce a
lightweight residual adapter $A_\theta$ that maps $h_t$ to a vocabulary-sized
logit correction. The defended training logits are
\begin{equation}
z_t = \tilde z_t + A_\theta(h_t),
\qquad
q_t = \mathrm{softmax}(z_t).
\end{equation}
The teacher backbone and the language-modeling head remain frozen throughout
training; only the adapter parameters $\theta$ are updated. This design keeps the
defense localized and avoids changing the underlying teacher representation.

\textbf{Teacher anchors.}
Since the adapter directly modifies the teacher's next-token distribution, it must be constrained to preserve trace quality. Without such constraints, anti-distillation pressure can move probability mass toward locally confusing but linguistically poor continuations; in preliminary experiments, this made traces harder to imitate but sometimes introduced repetition or unnatural fragments. This resembles the unreliable-tail problem in neural text generation, where low-probability tokens degrade coherence~\citep{holtzman2019curious}. We therefore restrict the logits to top-$k$ supports and renormalize the resulting distributions before computing divergences, and anchor the adapted distribution to the base final-layer prediction and an intermediate-layer projection of the backbone.

\textit{The first anchor} preserves the final-layer teacher distribution. Let $\mathcal{K}_t^{T}$ be the top-$k$ tokens under the base logits $\tilde z_t$, and let $\Pi_{\mathcal{K}}(\cdot)$ denote the distribution obtained by restricting logits to support $\mathcal{K}$ and renormalizing. We penalize deviations from the original teacher on this support:
\begin{equation}
\mathcal{L}_{\mathrm{base}}
=
\frac{1}{N}
\sum_t
\mathrm{KL}\!\left(
\Pi_{\mathcal{K}_t^T}(\tilde z_t)
\,\|\,
\Pi_{\mathcal{K}_t^T}(z_t)
\right).
\label{eq:base_loss}
\end{equation}
We use the directional KL divergence because $\mathcal{L}_{\mathrm{base}}$ is a preservation anchor: the frozen final-layer distribution defines which high-probability continuations should remain stable under the adapter. This term discourages the adapter from suppressing tokens that the base teacher assigns high probability, thereby preserving task utility and local fluency.

\textit{The second anchor} is an intermediate-layer consistency regularizer. This term is inspired by the layer-contrastive view in ~\citep{chuang2023dola}: different transformer depths expose different predictive views of the next token. The final layer is the teacher's mature prediction and is usually where task-specific and factual information is most strongly expressed. In contrast, an intermediate layer is closer to the input-side computation and often reflects a smoother, less specialized continuation preference: it may already support fluent local tokens before the final layers sharpen the distribution toward more fact-specific choices. We therefore use the intermediate projection to keep the generated trace natural. It guides the adapter to make anti-distillation changes without moving too far away from the teacher's own fluent continuations. 

Let $h_t^{(m)}$ be a fixed intermediate hidden state, and let $\tilde{z}_t^{(m)} = W_{\mathrm{lm}}h_t^{(m)}$ denote its projection through the frozen LM head. On the top-$k$ support $\mathcal{K}_t^m$ selected from $\tilde{z}_t^{(m)}$, we add
\begin{equation}
\mathcal{L}_{\mathrm{mid}}
=
\frac{1}{N}
\sum_t
\mathrm{JS}\!\left(
\Pi_{\mathcal{K}_t^m}(z_t),
\Pi_{\mathcal{K}_t^m}(\tilde z_t^{(m)})
\right).
\label{eq:mid_loss}
\end{equation}

Here, $\mathrm{JS}$ denotes the Jensen-Shannon divergence, a symmetric distance between two normalized token distributions. 
We use a symmetric divergence because the intermediate projection is not a target to imitate exactly; rather, it provides a second teacher-internal view that the adapted final prediction should remain compatible with.
This helps preserve fluency under anti-distillation pressure; empirically, removing $\mathcal{L}_{\mathrm{mid}}$ lowers the fluency curve in Figure~\ref{fig:gsm8k_ablation}(b).
Intuitively, $\mathcal{L}_{\mathrm{mid}}$ pulls the defended distribution toward the intermediate layer's smoother continuation preferences, while $\mathcal{L}_{\mathrm{base}}$ keeps it anchored to the final teacher prediction. Together, the complete utility objective is
\begin{equation}
\mathcal{L}_{\mathrm{util}}
=
\mathcal{L}_{\mathrm{base}}
+
\mathcal{L}_{\mathrm{mid}}.
\label{eq:util_loss}
\end{equation}

\subsection{Proxy-guided anti-distillation}
\label{sec:inner_anti}

The defended-teacher objective above requires a training signal indicating which
parts of a reasoning trace are most useful for imitation. Because the actual downstream student is unknown, we introduce a proxy student $S_{\phi}$ only during adapter training. The proxy is not meant to approximate the attacker exactly; instead, it acts as a local probe for identifying next-token positions where imitation is fragile.

For the current batch of teacher traces, we evaluate the proxy student using the
same next-token imitation loss $\mathcal{L}_{\mathrm{CE}}(\phi)$ defined above.
Rather than taking a descent step that improves the proxy on these traces, we
consider the opposite local problem. For a small radius $\rho$, define the ideal
loss-increased proxy as
\[
\phi^\star
\in
\mathrm{argmax}_{\phi':\,\|\phi'-\phi\|_2\le\rho}
\mathcal{L}_{\mathrm{CE}}(\phi').
\]
Since solving this constrained maximization exactly is expensive, we approximate
it by the first-order maximizer within the $\ell_2$ ball and define the virtual
proxy
\begin{equation}
\phi^+
=
\phi
+
\rho
\frac{
\nabla_\phi \mathcal{L}_{\mathrm{CE}}(\phi)
}{
\left\|
\nabla_\phi \mathcal{L}_{\mathrm{CE}}(\phi)
\right\|_2
},
\label{eq:virtual_student}
\end{equation}
where $\rho$ is the local perturbation radius. The parameter $\phi^+$ is not committed to the proxy
model; it is used only as a stop-gradient probe for the outer adapter update.

Given the virtual proxy, we evaluate it on the same teacher traces. Since
$\phi^+$ approximates a local maximizer of $\mathcal{L}_{\mathrm{CE}}$, it
defines a nearby proxy state under which the current traces are harder to
imitate. Let $u^{+}_t=S_{\phi^+}(x_{\le t})$ denote the next-token logits of this locally loss-maximized proxy. 

Using the same support-restriction principle as above, we compute the anti-distillation loss only on a proxy-sensitive top-$k$ support. Let $\mathcal{K}_t^{S}$ be the top-$k$ tokens under $u_t^+$. We define
\begin{equation}
\mathcal{L}_{\mathrm{anti}}
=
\frac{
\sum_t w_t^{S}\,
\mathrm{KL}\!\left(
\Pi_{\mathcal{K}_t^S}(u_t^+)
\,\|\, 
\Pi_{\mathcal{K}_t^S}(z_t)
\right)
}{
\sum_t w_t^{S}
}, \quad \text{with} \quad w_t^{S} = \frac{\mathrm{Entropy}\left(\operatorname{softmax}(u_t^+)\right)}{\log |V|}
\label{eq:anti_loss}
\end{equation}

where $w_t^S$ is an entropy-based weight computed from the virtual proxy distribution. Because $u_t^+$ comes from a nearby proxy state that locally increases the imitation loss, $\Pi_{\mathcal{K}_t^S}(u_t^+)$ provides a local anti-distillation target that is intentionally biased away from the ground-truth next-token imitation direction. Minimizing this KL moves the defended logits toward
proxy-sensitive alternatives on the restricted support, concentrating pressure on positions where imitation is unstable. Higher-entropy positions receive larger weights because the proxy is less confident there, leaving more room for the defended teacher to alter the imitation signal; easy or deterministic continuations are correspondingly
downweighted.

This proxy-guided loss can be viewed as a training-time analogue of ADS. Whereas ADS injects a proxy-derived perturbation during decoding, we instead train the adapter so that a related anti-distillation direction is already encoded in the teacher's logits before generation begins.

\subsection{Conflict-aware optimization}
\label{sec:conflict}

The utility and anti-distillation objectives have training-dependent alignment, making a fixed loss weight brittle. 
When their gradients are aligned, the utility update may already preserve part of the anti-distillation direction; when they conflict, the utility update can erase it. 
We therefore combine the two gradients with an adaptive coefficient based on their alignment:
\begin{equation}
  \nabla_\theta \mathcal{L}
  \;=\;
  \nabla_\theta \mathcal{L}_{\mathrm{util}}
  + \lambda\, \nabla_\theta \mathcal{L}_{\mathrm{anti}},
  \label{eq:final_grad}
\end{equation}
where the adaptive coefficient $\lambda$ is computed as:
\begin{equation}
  \lambda
  \;=\;
  \max\!\left(
    0,\;
    \kappa
    - \frac{\left\langle \nabla_\theta \mathcal{L}_{\mathrm{util}}, \nabla_\theta \mathcal{L}_{\mathrm{anti}} \right\rangle}
           {\left\lVert \nabla_\theta \mathcal{L}_{\mathrm{anti}} \right\rVert_2^2 + \varepsilon}
  \right).
  \label{eq:adaptive_lambda}
\end{equation}
Here, $\kappa>0$ is a margin hyperparameter, and $\lambda$ is detached from the computation graph.
This rule can be interpreted as maintaining a minimum projection of the final update onto the anti-distillation direction. The quantity 
$\frac{\left\langle \nabla_\theta \mathcal{L}_{\mathrm{util}}, \nabla_\theta \mathcal{L}_{\mathrm{anti}} \right\rangle}
           {\left\lVert \nabla_\theta \mathcal{L}_{\mathrm{anti}} \right\rVert_2^2 + \varepsilon}
$ 
measures how much of the anti-distillation direction is already present in the utility gradient. When the gradients are aligned, this projection is large, so the additional anti-distillation weight is reduced. When the gradients are orthogonal or opposed, the projection is small or negative, so $\lambda$ increases and adds back anti-distillation pressure. This prevents the utility regularizers from washing out the protective update while still allowing them to dominate when they already support the same direction.
We provide a first-order constrained-projection interpretation of this update in Appendix~\ref{app:conflict_aware_analysis}.

\subsection{Generation-time strength control}
\label{sec:generation_control}
During training, the objectives are defined on the full adapted distribution $q_t$. At generation time, we expose the learned residual with a scalar $\delta\in[0,1]$, which controls the strength of the defense by interpolating between the original and adapted logits:
\begin{equation}
\label{eq:delta-interpolation}
z_t^{(\delta)}
=
(1-\delta)\tilde z_t + \delta  z_t
=
\tilde z_t + \delta A_\theta(h_t),
\qquad
q_t^{(\delta)}
=
\mathrm{softmax}\!\left(z_t^{(\delta)}\right).
\end{equation}
Thus, $\delta=0$ recovers the original teacher exactly, while $\delta=1$ applies the learned defense at full scale. Intermediate values provide a practical way to navigate the utility-distillability trade-off without retraining the
adapter. After interpolation, generation uses standard temperature and top-$k$ sampling.

The full procedure is shown in Algorithm~\ref{alg:procedure}.

\begin{algorithm}[htbp]
\caption{Conflict-Aware Logit Adapter}
\label{alg:procedure}
\textbf{Input:} Frozen teacher $T_\Theta$, adapter $A_\theta$, proxy $S_\phi$,
dataset $\mathcal{D}$, support size $k$, conflict margin $\kappa$, inference strength $\delta\in[0,1]$

\textbf{Training.} For each minibatch $\mathcal{B}$:
\begin{enumerate}
  \item Compute adapted logits $z_t \leftarrow \tilde{z}_t + A_\theta(h_t)$
        and accumulate proxy gradients $\nabla_\phi \mathcal{L}_{\mathrm{CE}}(\phi)$.
 \item Take a virtual proxy ascent step
        $\phi^+=\phi+\rho\frac{\nabla_\phi \mathcal{L}_{\mathrm{CE}}(\phi)}{\left\|\nabla_\phi \mathcal{L}_{\mathrm{CE}}(\phi)\right\|_2},$
        and compute the anti-distillation loss $\mathcal{L}_{\mathrm{anti}} $ \eqref{eq:anti_loss}
        on the proxy-sensitive support $\mathcal{K}^S(\phi^+)$.
\item Compute the utility loss $\mathcal{L}_{\mathrm{util}} = \mathcal{L}_{\mathrm{base}}(\mathcal{K}^T)
        + \mathcal{L}_{\mathrm{mid}}(\mathcal{K}^m)$ \eqref{eq:util_loss},
        then update $\theta$ via
        \[
          \nabla_\theta \mathcal{L} = \nabla_\theta \mathcal{L}_{\mathrm{util}}
          + \lambda\, \nabla_\theta \mathcal{L}_{\mathrm{anti}},
        \]
        where $\lambda$ in \eqref{eq:adaptive_lambda} down-weights $\mathcal{L}_{\mathrm{anti}}$
        when the two gradients are aligned and amplifies it when they conflict.
        Discard $\phi^+$.
\end{enumerate}

\textbf{Inference.} At each step $t$, apply the learned residual with strength $\delta$
and sample the next token:
\[
  x_{t+1} \sim \mathrm{softmax}\!\left(\tilde{z}_t + \delta\, A_\theta(h_t)\right).
\]
\textbf{Output:} Generated sequence with tunable distillation resistance.
\end{algorithm}

\section{Experiments}
\label{sec:experiments}
We evaluate whether a training-time residual adapter improves the utility-distillability trade-off over decoding-time antidistillation, while avoiding deployment-time proxy-student computation. For each operating point, we report two metrics: \textit{defended teacher accuracy} (user-facing utility) and \textit{downstream student accuracy} (residual distillation signal). At a fixed teacher accuracy, lower student accuracy indicates stronger protection.

\subsection{Experimental Setup}
\label{sec:experimental_setup}

\textbf{Models and Training Data.}
We use DeepSeek-R1-Distill-Qwen-7B~\cite{deepseekai2025deepseekr1} as the defended teacher.
We use Qwen2.5-1.5B~\cite{yang2025qwen25} as the training-time proxy student and Llama-3.2-3B~\cite{meta2024llama32modelcard} as the downstream student for trace-distillation evaluation.

We train the adapter on 6{,}800 MetaMathQA~\cite{yu2024metamath} samples with teacher-generated responses.

\textbf{Baselines.}
Our main trade-off curves compare against two baselines. \textbf{Antidistillation Sampling} (ADS)~\citep{savani2025antidistillation} applies a proxy-student-derived defense at decoding time.
 And \textbf{Temperature Sampling} varies only the teacher sampling temperature $\tau$ with no adapter and no proxy-derived antidistillation penalty, providing a simple control for ordinary sampling noise.

\textbf{Evaluation protocol.} Our setup closely follows the baseline protocol outlined by ADS~\citep{savani2025antidistillation}. We evaluate on three benchmarks: GSM8K~\cite{cobbe2021training}, MATH~\cite{hendrycks2021measuring}, and MMLU~\cite{hendrycks2020measuring} (the adapter-training dataset and all evaluation benchmarks are provided under the MIT License). For each benchmark, traces generated by the defended teacher are used to train a downstream student, ensuring a strictly consistent distillation procedure across all evaluated methods. 

For each method, we evaluate four operating points across different hyperparameter configurations, visualizing the results as trade-off curves. Exact numerical values and point-specific configurations are detailed in Appendix~\ref{app:additional_experiments}. Implementation details and hyperparameters for training and evaluation are provided in Appendix~\ref{app:hyperparameters}.

\subsection{Results \& Discussion}
\label{sec:results}

\begin{figure}[t]
    \centering

    \begin{minipage}[t]{0.75\linewidth}
        \centering
        \includegraphics[width=\linewidth]{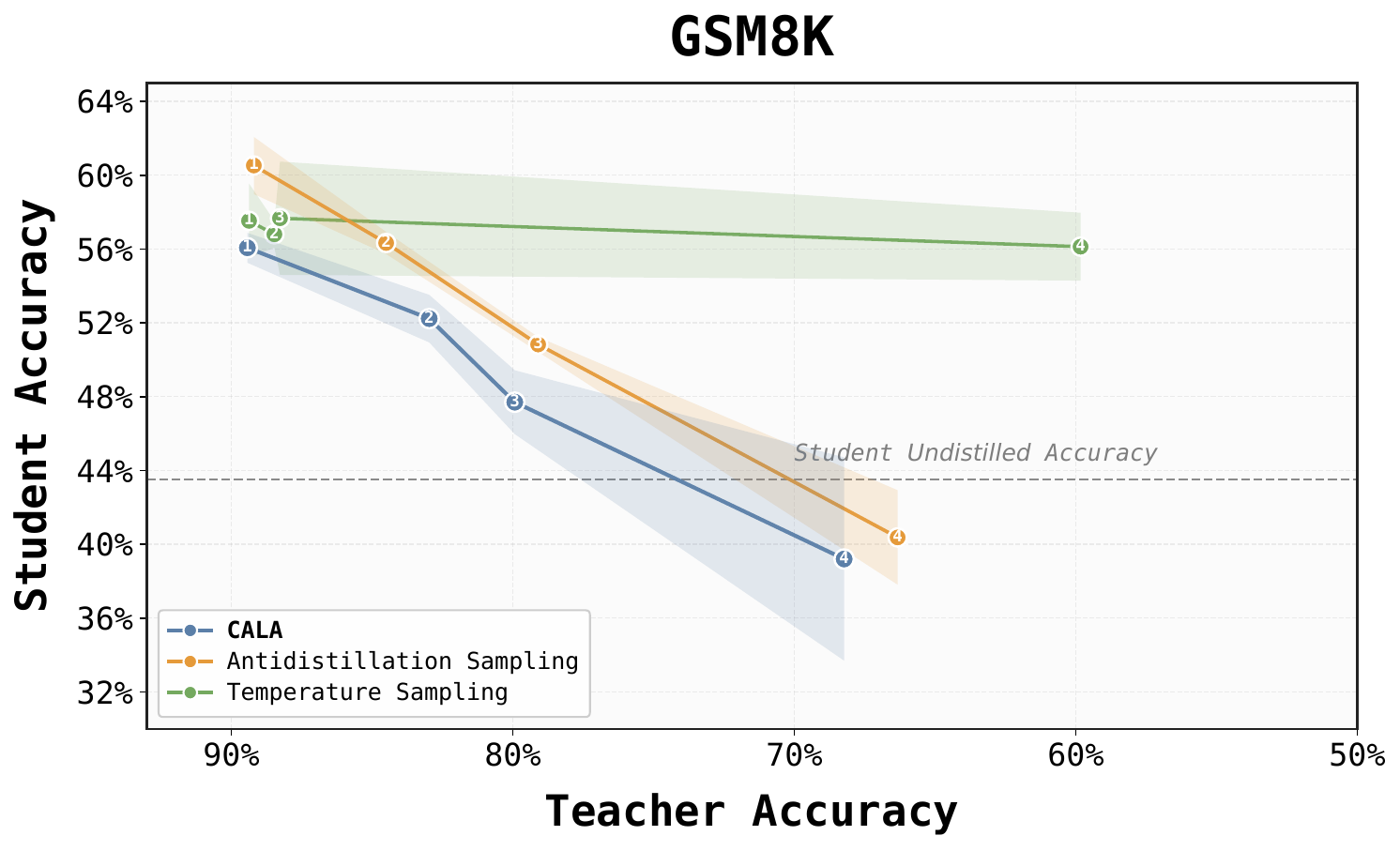}
        \vspace{-0.5em}
        \centerline{\small (a) GSM8K}
    \end{minipage}

    \vspace{0.8em}

    \begin{minipage}[t]{0.49\linewidth}
        \centering
        \includegraphics[width=\linewidth]{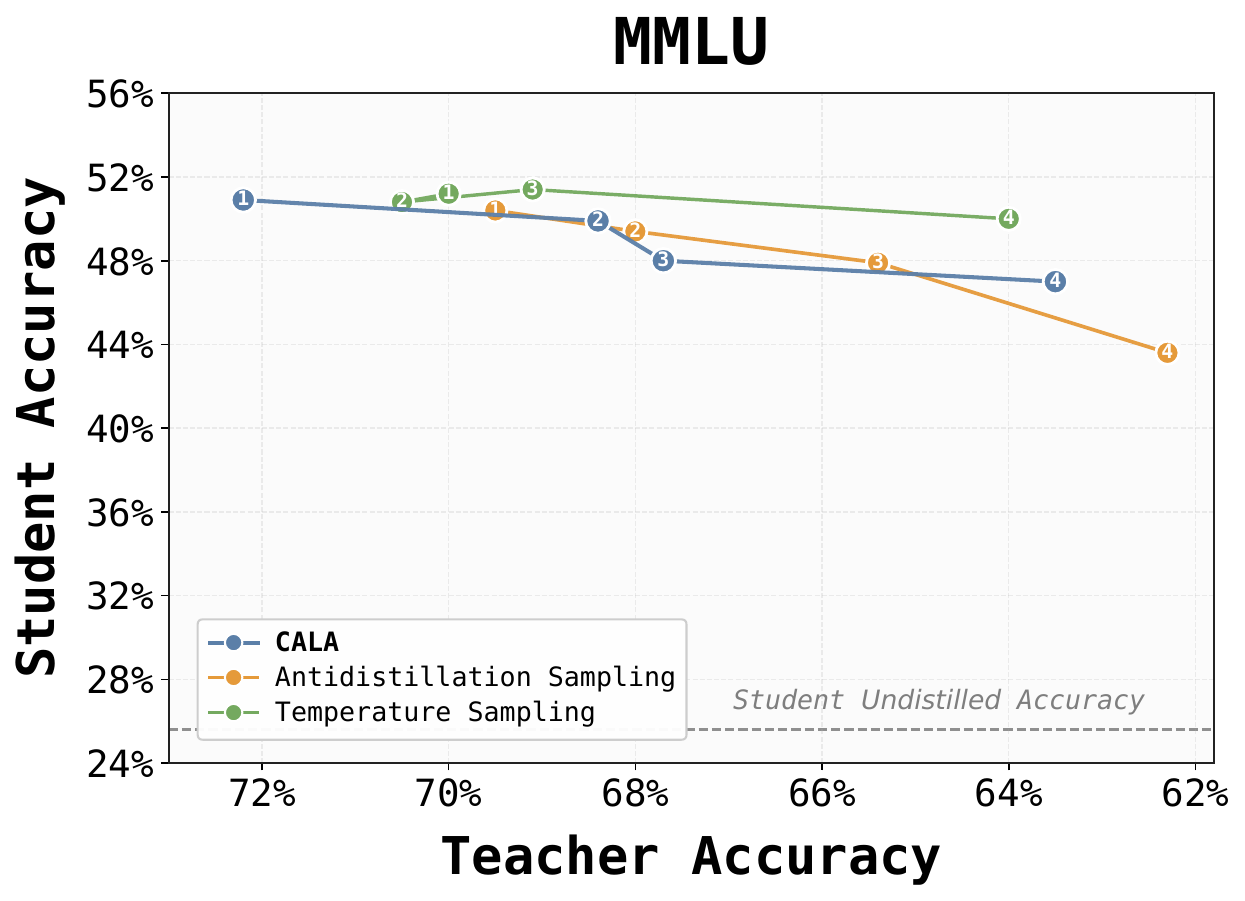}
        \vspace{-0.5em}
        \centerline{\small (b) MMLU}
    \end{minipage}
    \hfill
    \begin{minipage}[t]{0.49\linewidth}
        \centering
        \includegraphics[width=\linewidth]{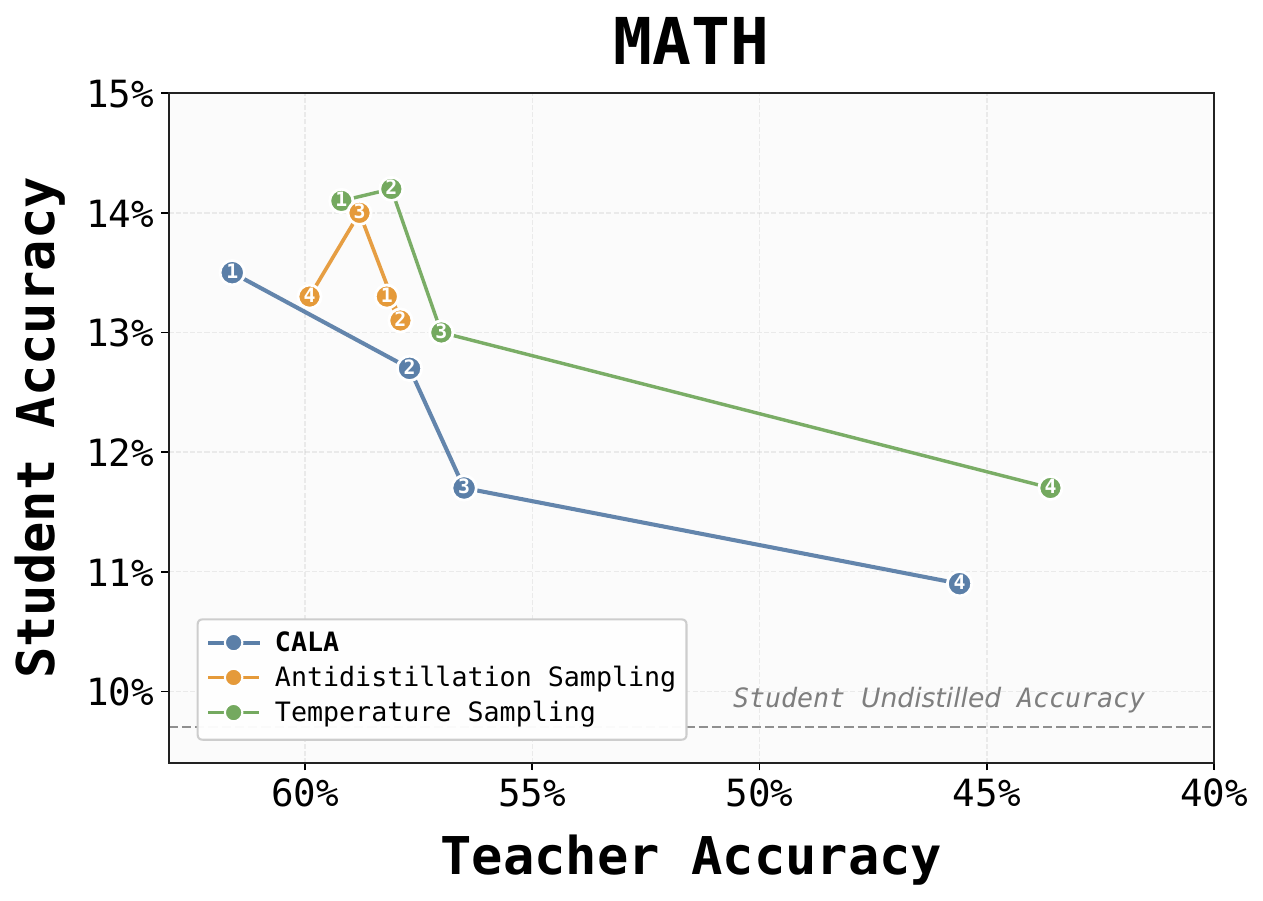}
        \vspace{-1em}
        \centerline{\small (c) MATH}
    \end{minipage}

    \caption{Utility-distillability trade-offs across three benchmarks. The x-axis reports defended teacher accuracy, while the y-axis reports the accuracy of a downstream student distilled from the generated traces. Lower student accuracy at comparable teacher accuracy indicates stronger protection. For GSM8K, the lines show three-seed means and the shaded bands show $\pm$ one standard deviation in student accuracy; MMLU and MATH report the original evaluation runs.}
    \vspace{-1em}
    \label{fig:tradeoff}
\end{figure}

\textbf{Utility-Distillability Trade-offs.}
As illustrated in Figure~\ref{fig:tradeoff}(a), CALA achieves the strongest utility-distillability trade-off on GSM8K among the evaluated methods. At comparable defended-teacher accuracies, CALA consistently yields lower mean student accuracy than both ADS and temperature sampling, indicating stronger protection against trace-based distillation. For instance, at approximately 79--80\% mean teacher accuracy, the distilled student reaches 50.8\% mean accuracy under ADS but drops to 47.7\% under CALA. While temperature sampling also degrades student accuracy, it follows a sub-optimal trade-off curve, indicating that ordinary sampling noise does not account for the protection provided by our learned adapter.

These trade-offs extend to MATH and MMLU (Figure~\ref{fig:tradeoff}b, c). On MATH, CALA is particularly effective in the critical high-utility regime, aggressively reducing student accuracy while preserving high teacher performance. On MMLU, the margins between methods are narrower. While CALA still improves over temperature sampling in the high- and mid-utility regions, ADS achieves the lowest student accuracy at its most aggressive setting. However, as we further analyze in Section~\ref{sec:fluency_results}, ADS attains this extreme reduction in distillability at the cost of severe surface-level degeneration.

\textbf{Fluency and Degeneration Analysis.}
We further evaluate surface-level generation quality using GPT-4o mini as an automatic fluency judge, which scores grammaticality, coherence, readability, and local repetition while ignoring mathematical correctness and the final answer; see Appendix~\ref{app:fluency_judge} for the full prompt and scoring protocol.
This evaluation is important because a defense that reduces student accuracy by making traces unreadable would not preserve the utility we aim to protect.

Figure~\ref{fig:fluency} evaluates whether lower downstream-student accuracy is accompanied by surface-level degeneration. On GSM8K and MMLU, CALA maintains higher fluency than ADS across the evaluated operating points, with ADS showing a sharper fluency drop as defense strength increases. Together with the accuracy curves, these results suggest that the high-utility reductions in distillability are not solely explained by surface degeneration.
\label{sec:fluency_results}

\begin{figure}[t]

    \begin{minipage}[t]{0.49\linewidth}
        \centering
        \includegraphics[width=\linewidth]{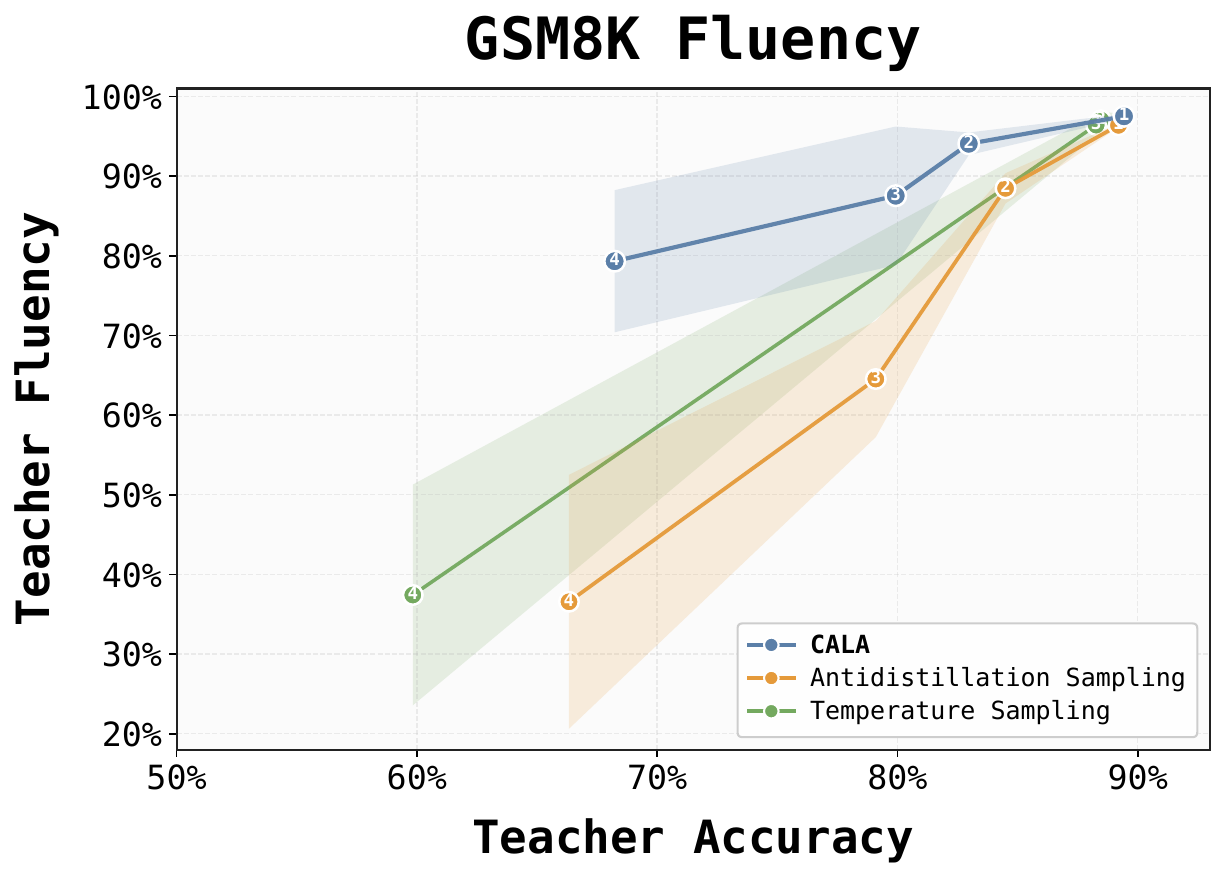}
        \vspace{-0.5em}
        \centerline{\small (a) GSM8K}
    \end{minipage}
    \hfill
    \begin{minipage}[t]{0.49\linewidth}
        \centering
        \includegraphics[width=\linewidth]{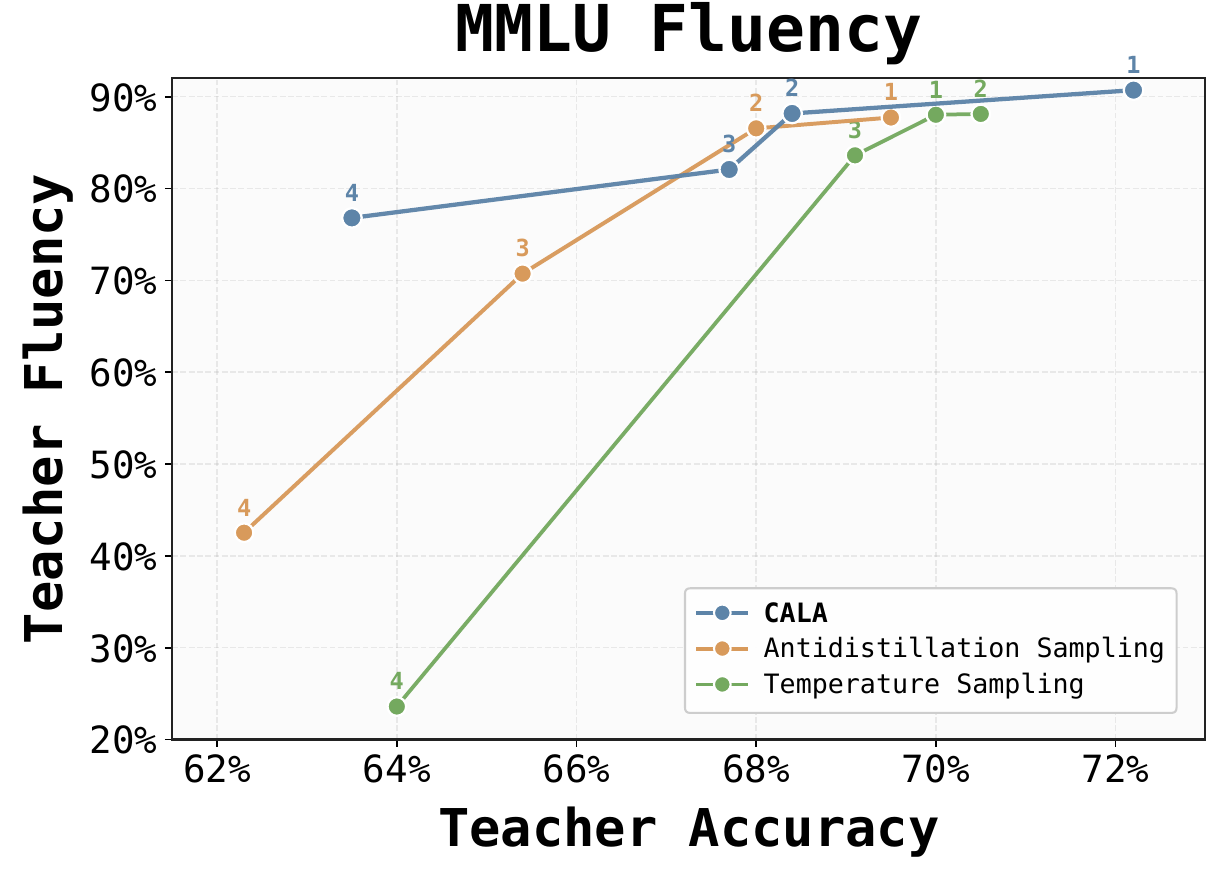}
        \vspace{-0.5em}
        \centerline{\small (b) MMLU}
    \end{minipage}

    \caption{Fluency evaluation using GPT-4o mini as an automatic judge. The judge scores surface quality while ignoring mathematical correctness and the final answer. For GSM8K, the lines show three-seed means and the shaded bands show $\pm$ one standard deviation in fluency; MMLU reports the original evaluation run.}
    \label{fig:fluency}
    \vspace{-0.5em}
\end{figure}

\textbf{Inference efficiency.}
We benchmark inference costs on a single H100 GPU using a batch size of 64 and a maximum sequence length of 1024 tokens. As shown in Table~\ref{tab:inference_cost}, CALA stays close to the base teacher's inference cost, whereas ADS is substantially slower and consumes nearly twice the peak memory due to its online proxy-model computation. This supports the intended shift from decoding-time intervention to a lightweight training-time adapter.

\begin{table}[h]
\centering
\small
\setlength{\tabcolsep}{6pt}
\caption{Comparison of inference costs and memory footprints. Memory is reported in GB.}
\label{tab:inference_cost}
\begin{tabular}{lcccc}
\toprule
Method & Time / batch & Tokens / s & Avg. memory & Peak memory \\
\midrule
Base teacher & 25.8s & 828.5 & 15.26 & 20.07 \\
ADS & 89.2s & 465.1 & 32.39 & 39.28 \\
CALA & 31.5s & 708.3 & 15.88 & 20.24 \\
\bottomrule
\end{tabular}
\end{table}

\textbf{Proxy sensitivity.}
Because our method uses a proxy only during adapter training, we test whether the learned defense is sensitive to proxy choice. 
We observe that replacing the default 1.5B proxy with a 3B proxy from the same tokenizer-compatible family preserves the qualitative utility-distillability trade-off on GSM8K, as shown in Figure~\ref{fig:gsm8k_proxy_sensitivity} in Appendix~\ref{app:additional_experiments}. This study establishes robustness to proxy size within a compatible token space, not to arbitrary proxy families or tokenizers.

\begin{figure*}[t]
    \centering
    \begin{minipage}[t]{0.59\textwidth}
        \vspace{0pt}
        \centering
        \includegraphics[width=\linewidth]{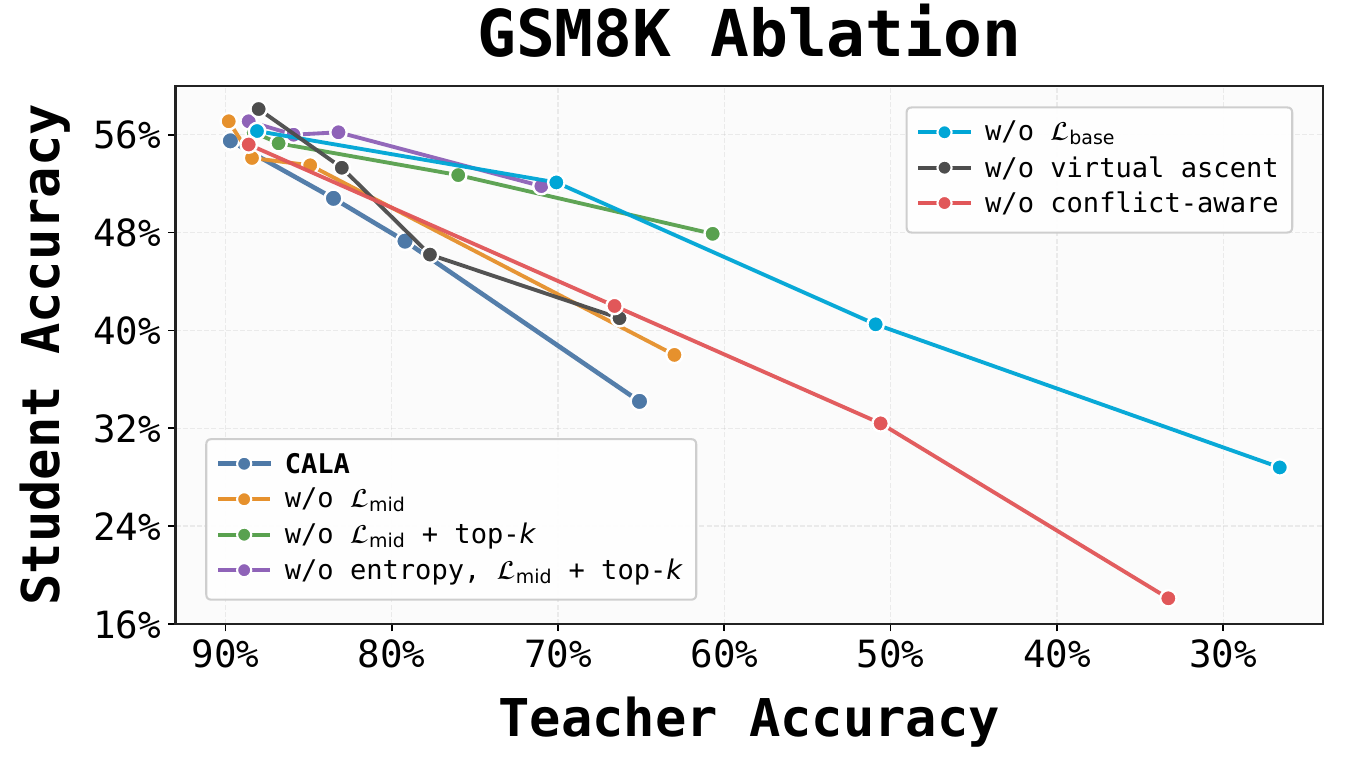}
        \vspace{-0.4em}
        \centerline{\small (a) Utility-distillability trade-off}
    \end{minipage}
    \hfill
    \begin{minipage}[t]{0.39\textwidth}
        \vspace{0pt}
        \centering
        \includegraphics[width=\linewidth]{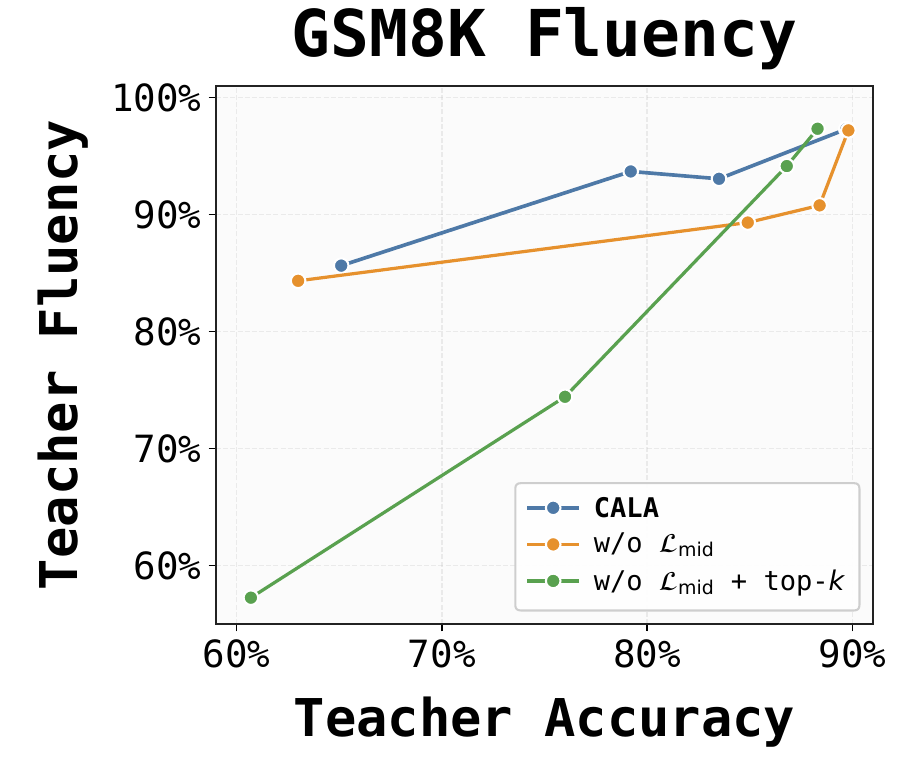}
        \vspace{-0.4em}
        \centerline{\small (b) Teacher fluency}
    \end{minipage}
    \caption{GSM8K ablations. Panel (a) reports the utility-distillability trade-off for objective and support variants as well as individual removals of the base anchor, virtual ascent, and conflict-aware update. Panel (b) reports teacher fluency for variants with available fluency measurements. Lower student accuracy at comparable teacher accuracy indicates stronger protection.}
    \label{fig:gsm8k_ablation}
    \vspace{-1em}
\end{figure*}
\textbf{Ablations.}
We ablate the main objective and optimization components on GSM8K. Figure~\ref{fig:gsm8k_ablation}(a) shows that the full objective gives the best overall trade-off. Removing $\mathcal{L}_{\mathrm{mid}}$ makes the results less consistent, while jointly removing $\mathcal{L}_{\mathrm{mid}}$ and the top-$k$ restriction further weakens protection. Removing entropy weighting as well leaves higher downstream-student accuracy, suggesting that these components jointly focus the adapter update on useful anti-distillation directions. Among the individual removals, omitting $\mathcal{L}_{\mathrm{base}}$ sharply reduces teacher utility. At comparable 65--67\% teacher accuracy, removing virtual ascent or the conflict-aware update raises student accuracy from 34.2\% to 41.0\% and 42.0\%, respectively.

Figure~\ref{fig:gsm8k_ablation}(b) shows that jointly removing $\mathcal{L}_{\mathrm{mid}}$ and the top-$k$ restriction reduces fluency under stronger defense settings. Exact operating points and available fluency values are reported in Appendix~\ref{app:additional_experiments}.

\textbf{Robustness and transfer.}
We additionally audit reasoning validity, vary the downstream student, evaluate a cross-task distribution shift, and test an adaptive filtering attack; additional protocols and results are in Appendix~\ref{app:rebuttal_experiments}. In a blinded audit of 100 GSM8K traces at approximately matched teacher accuracy, 36/50 CALA traces and 31/50 ADS traces were strictly valid. The paired difference is not statistically significant (exact McNemar $p=0.332$); within this audit, we did not observe a higher invalid-rationale rate for CALA, and we do not claim superiority in reasoning correctness.

Protection transfers from the Qwen training proxy to Gemma-2-2B-IT and remains competitive for Qwen2.5-3B in the high-utility regime. Under the MetaMathQA$\rightarrow$MMLU-traces$\rightarrow$GPQA shift, CALA lowers student accuracy from 31.2\% to 27.2\% relative to ADS at exactly matched 37.9\% teacher accuracy, while the methods are effectively tied at a second matched point.

\textbf{Adaptive perplexity filtering.}
We also test an adaptive attacker that uses the student's pretrained base model to discard the 20\% highest-perplexity traces before fine-tuning.
This experiment uses the original single-run operating points reported in Table~\ref{tab:exact_acc_all}.

\begin{table}[htbp]
\centering
\small
\setlength{\tabcolsep}{9pt}
\caption{Adaptive PPL-filtering attack on GSM8K. Lower student accuracy indicates stronger protection.}
\label{tab:ppl_filtering}
\begin{tabular}{cccccc}
\toprule
\multicolumn{2}{c}{\shortstack{Teacher acc.$\uparrow$\\\phantom{Base}}}
& \multicolumn{2}{c}{\shortstack{Student acc. $\downarrow$\\No filtering}}
& \multicolumn{2}{c}{\shortstack{Student acc. $\downarrow$\\PPL filtering}} \\
\cmidrule(lr){1-2}\cmidrule(lr){3-4}\cmidrule(lr){5-6}
CALA & ADS & CALA & ADS & CALA & ADS \\
\midrule
89.7 & 89.7 & 55.5 & 61.0 & 56.4 & 59.6 \\
83.5 & 84.0 & 50.8 & 55.7 & 54.0 & 56.7 \\
79.2 & 78.9 & 47.3 & 50.4 & 51.8 & 50.0 \\
65.1 & 64.8 & 34.2 & 38.6 & 44.6 & 41.3 \\
\bottomrule
\end{tabular}
\end{table}

As shown in Table~\ref{tab:ppl_filtering}, filtering partially recovers student performance. CALA remains stronger around 90\% and 84\% teacher accuracy, whereas ADS is stronger at the two lower-utility settings. Separately, across three seeds, CALA yields lower mean student accuracy than ADS at the first three operating points; variability is much larger at the most aggressive setting. These results support transfer and initial robustness only in the evaluated settings, rather than against arbitrary students or adaptive attacks; paraphrasing- and verifier-based attacks remain unevaluated.

\section{Conclusion}
\label{sec:conclusion}

We presented CALA, a training-time defense for protecting reasoning traces against black-box distillation. By training a lightweight residual logit adapter on top of a frozen teacher, CALA amortizes proxy-guided anti-distillation pressure into the model and avoids proxy-student computation during decoding. Across reasoning benchmarks, CALA reduces downstream student imitation while preserving teacher utility in high-utility regimes, with better surface fluency than ADS on GSM8K.

CALA is limited to the studied black-box trace-distillation setting and should not be interpreted as protection against all forms of model extraction or adaptive distillation. Its effectiveness may depend on the training-time proxy, and the current proxy-size study assumes a compatible token space. Although we evaluate cross-student transfer and perplexity-based filtering, paraphrasing, verifier-based attacks, and tokenizer-mismatched training proxies remain open. Aggressive settings can also reduce teacher accuracy or fluency. More broadly, trace-modification defenses may help reduce unauthorized imitation but should be deployed carefully to avoid reducing readability or user trust.

\section*{Acknowledgments}
This research was funded by Vingroup Innovation Foundation (VINIF) under project code VINIF.2026.DA117.

\newpage
\bibliographystyle{plainnat}
\bibliography{reference}

\newpage
\appendix

\section{Analysis of the Conflict-Aware Update}
\label{app:conflict_aware_analysis}

This section gives a first-order interpretation of the conflict-aware optimization rule. The projection view is similar in spirit to the constrained first-order update of~\citep{wu2024erasediff}, which balances preservation and erasure by choosing an update direction that stays close to a utility gradient while satisfying a constraint on the forgetting objective. In our setting, the adapter is trained with the utility-preserving objective \(\mathcal{L}_{\mathrm{util}}\), which anchors the adapted distribution to the frozen teacher, and the anti-distillation objective \(\mathcal{L}_{\mathrm{anti}}\), which moves the adapted distribution toward regions that are less useful for a proxy student. Because these objectives can induce conflicting gradients, the update combines them adaptively:
\[
\nabla_\theta \mathcal{L}
=
\nabla_\theta \mathcal{L}_{\mathrm{util}}
+
\lambda \nabla_\theta \mathcal{L}_{\mathrm{anti}},
\]
where
\[
\lambda
=
\max\left(
0,\,
\kappa
-
\frac{
\left\langle
\nabla_\theta \mathcal{L}_{\mathrm{util}},
\nabla_\theta \mathcal{L}_{\mathrm{anti}}
\right\rangle
}{
\left\|\nabla_\theta \mathcal{L}_{\mathrm{anti}}\right\|_2^2+\epsilon
}
\right).
\]
We show that, up to the numerical stabilizer \(\epsilon\), this rule is the closed-form solution of a constrained projection problem.

Let
\[
g_u = \nabla_\theta \mathcal{L}_{\mathrm{util}}(\theta),
\qquad
g_a = \nabla_\theta \mathcal{L}_{\mathrm{anti}}(\theta).
\]
For a gradient-descent step \(\theta^+=\theta-\eta d\), the utility-only choice is \(d=g_u\). However, if \(g_u\) is poorly aligned with \(g_a\), this update can weaken the anti-distillation signal. We therefore choose the direction closest to \(g_u\) while enforcing a minimum first-order decrease in \(\mathcal{L}_{\mathrm{anti}}\).

\begin{theorem}[Conflict-aware projection]
\label{thm:conflict_projection}
Let \(g_u,g_a\in\mathbb{R}^p\), with \(g_a\neq 0\), and let \(\kappa>0\). Consider the constrained projection problem:
\[
d^\star
=
\arg\min_d \frac{1}{2}\|d-g_u\|_2^2
\quad
\mathrm{s.t.}
\quad
\langle g_a,d\rangle \ge \kappa\|g_a\|_2^2 .
\]
The unique solution is
\[
d^\star = g_u+\lambda^\star g_a,
\qquad
\lambda^\star
=
\max\left\{
0,\,
\kappa-
\frac{\langle g_u,g_a\rangle}{\|g_a\|_2^2}
\right\}.
\]
\end{theorem}

\begin{proof}
The objective is strongly convex in \(d\), and the feasible set is a closed half-space, so the projection problem has a unique optimum. Consider the Lagrangian
\[
\mathcal{J}(d,\lambda)
=
\frac{1}{2}\|d-g_u\|_2^2
+
\lambda\left(\kappa\|g_a\|_2^2-\langle g_a,d\rangle\right),
\qquad \lambda\ge 0.
\]
The stationarity condition gives
\[
\nabla_d\mathcal{J}(d,\lambda)=d-g_u-\lambda g_a=0,
\]
and hence \(d=g_u+\lambda g_a\). If \(\langle g_a,g_u\rangle\ge \kappa\|g_a\|_2^2\), the constraint is inactive, so \(\lambda^\star=0\) and \(d^\star=g_u\). Otherwise, the constraint is active, and substituting \(d=g_u+\lambda g_a\) into \(\langle g_a,d\rangle=\kappa\|g_a\|_2^2\) gives
\[
\langle g_a,g_u\rangle+\lambda\|g_a\|_2^2
=
\kappa\|g_a\|_2^2.
\]
Thus,
\[
\lambda
=
\kappa-
\frac{\langle g_u,g_a\rangle}{\|g_a\|_2^2}.
\]
Combining the active and inactive cases gives the claimed \(\lambda^\star\), and therefore \(d^\star=g_u+\lambda^\star g_a\).
\end{proof}

In practice, we replace \(\|g_a\|_2^2\) by \(\|g_a\|_2^2+\epsilon\) to avoid instability when the anti-distillation gradient is very small.

\begin{corollary}[First-order anti-distillation progress]
\label{cor:first_order_progress}
Let \(d^\star\) be the direction from Theorem~\ref{thm:conflict_projection}. If \(\mathcal{L}_{\mathrm{anti}}\) is differentiable at \(\theta\), then
\[
\mathcal{L}_{\mathrm{anti}}(\theta-\eta d^\star)
\le
\mathcal{L}_{\mathrm{anti}}(\theta)
-
\eta\kappa\|g_a\|_2^2
+
O(\eta^2).
\]
\end{corollary}

\begin{proof}
By Taylor expansion,
\[
\mathcal{L}_{\mathrm{anti}}(\theta-\eta d^\star)
=
\mathcal{L}_{\mathrm{anti}}(\theta)
-
\eta\langle g_a,d^\star\rangle
+
O(\eta^2).
\]
The constraint in Theorem~\ref{thm:conflict_projection} gives \(\langle g_a,d^\star\rangle\ge \kappa\|g_a\|_2^2\), which proves the claim.
\end{proof}

\paragraph{Interpretation.}
The conflict-aware update is therefore not an arbitrary weighted sum. It follows the utility-preserving gradient when that direction already has sufficient anti-distillation alignment, and otherwise adds the smallest correction in the anti-distillation gradient direction needed to meet the margin. Thus, \(\lambda\) decreases when the gradients are aligned and increases when they are orthogonal or opposed, preventing the utility objective from washing out the anti-distillation signal.

\section{Experimental Details}
\label{app:hyperparameters}

\subsection{Models and Training Setup}
\label{app:models_training}

Table~\ref{tab:model_roles} summarizes the model roles used in our experiments. The
defended teacher is DeepSeek-R1-Distill-Qwen-7B. We freeze the teacher backbone
and the language-modeling head, and train only a residual logit adapter. The
training-time proxy is Qwen2.5-1.5B, with its token embeddings resized to match the
teacher vocabulary. Downstream trace-distillation is evaluated by training
Llama-3.2-3B students on generated traces. The ADS baseline uses a Qwen2.5-3B proxy
during decoding.

\begin{table}[H]
\centering
\small
\caption{Model roles used in the experiments.}
\label{tab:model_roles}
\begin{tabular}{lll}
\toprule
Role & Model & Usage \\
\midrule
Defended teacher & DeepSeek-R1-Distill-Qwen-7B & Frozen teacher \\
Training proxy & Qwen2.5-1.5B & Adapter-training signal \\
ADS proxy & Qwen2.5-3B & Decoding-time baseline \\
Downstream student & Llama-3.2-3B & Trace-distillation evaluation \\
\bottomrule
\end{tabular}
\end{table}

The model and data licenses differ and are respected in our use: DeepSeek-R1-Distill-Qwen-7B and MetaMathQA use the MIT License, Qwen2.5 uses Apache 2.0, and Llama 3.2 uses the Llama 3.2 Community License. The GSM8K, MATH, and MMLU evaluation datasets are distributed under the MIT License.

We train the adapter on 6,800 MetaMathQA samples with teacher-generated responses.
The proxy student is used only to construct the training signal: for each accumulation
window, we compute masked next-token cross-entropy gradients, take one virtual
gradient-ascent step, use the resulting proxy distribution to define the
anti-distillation target, and then discard the virtual proxy update.

\subsection{Adapter Training Hyperparameters}
\label{app:adapter_hparams}

All KL and JS terms are computed on restricted top-\(k\) supports with \(k=50\). This
keeps both anti-distillation and utility-preservation losses on plausible
high-probability tokens.
Table~\ref{tab:adapter_hparams} summarizes the full adapter training configuration.

\begin{table}[H]
\centering
\small
\caption{Residual adapter training hyperparameters.}
\label{tab:adapter_hparams}
\begin{tabular}{ll}
\toprule
Hyperparameter & Value \\
\midrule
Training samples & 6,800 \\
Epochs & 5 \\
Batch size & 8 \\
Proxy accumulation steps \(G\) & 2 \\
Optimizer & AdamW \\
Learning rate & \(3\times10^{-4}\) \\
Weight decay & \(3\times10^{-2}\) \\
Schedule / warmup ratio & Warmup cosine / 0.05 \\
Adapter gradient clipping & 1.0 \\
Maximum sequence / prompt length & 600 / 128 \\
Proxy virtual step radius \(\rho\) & \(10^{-3}\) \\
Proxy gradient clipping & 1.0 \\
Top-\(k\) support & 50 \\
Conflict margin \(\kappa\) & 0.5 \\
Numerical stability \(\epsilon\) & \(10^{-6}\) \\
Seed & 42 \\
\bottomrule
\end{tabular}
\end{table}

\subsection{Evaluation Details}
\label{app:eval_detail}
\paragraph{Generation.} Our evaluation across GSM8K, MATH, and MMLU involves generating traces from the defended teacher to train a downstream student. To construct the trade-off curves, we evaluate four operating points per method, which are summarized in Table~\ref{tab:main_operating_points}. Specifically, for CALA, we vary the adapter strength $\delta$ and sampling temperature $\tau$; for ADS, we vary $\tau$ and the antidistillation weight $\lambda_{\mathrm{ADS}}$; and for the temperature baseline, we solely vary $\tau$. Unless otherwise noted, trace generation for CALA utilizes top-$k$ sampling with $k=50$. Generation is capped at 1024 tokens across all three benchmarks.

\begin{table}[h]
\vspace{-0.5em}
\centering
\small
\setlength{\tabcolsep}{6pt}
\caption{Evaluation-time operating points used for CALA, ADS, and the temperature baseline in the main figures.}
\label{tab:main_operating_points}
\begin{tabular}{c c c c}
\toprule
Point & CALA $(\tau,\delta)$ & ADS $(\tau,\lambda_{\mathrm{ADS}})$ & Temp. $\tau$ \\
\midrule
1 & $(0.50, 0.25)$ & $(0.90, 1.58{\times}10^{-2})$ & $0.90$ \\
2 & $(0.75, 0.75)$ & $(0.90, 2.37{\times}10^{-2})$ & $1.16$ \\
3 & $(0.50, 0.80)$ & $(0.90, 3.16{\times}10^{-2})$ & $1.26$ \\
4 & $(0.50, 0.85)$ & $(1.00, 3.95{\times}10^{-2})$ & $1.50$ \\
\bottomrule
\end{tabular}
\end{table}

\paragraph{Trace Generation Prompts. } We use fixed system prompts when generating teacher traces. For GSM8K and MATH, we use the following math-reasoning prompt:

\begin{tcolorbox}[
title=Math trace-generation prompt,
colback=gray!3,
colframe=gray!45,
fonttitle=\bfseries,
breakable,
arc=1mm,
boxrule=0.4pt
]
You are a math teacher. You will be given a math problem and you will solve it step by step.

You will output your final solution like \verb|\boxed{ANSWER}|. Be sure to include relevant units within the brackets and fully evaluate arithmetic expressions.
\end{tcolorbox}

For MMLU, we use the following multiple-choice prompt:

\begin{tcolorbox}[
title=MMLU trace-generation prompt,
colback=gray!3,
colframe=gray!45,
fonttitle=\bfseries,
breakable,
arc=1mm,
boxrule=0.4pt
]
You are a teacher. You will be given a problem and you will solve it step by step.

You will output your final solution like \verb|\boxed{X}|, where X is exactly one capital letter corresponding to the correct option (A, B, C, or D).

Do not include any extra text, punctuation, or words inside the box.
\end{tcolorbox}

With answer forcing, we append a fixed prompt string immediately before the final short answer to enforce a consistent answer format:

\begin{tcolorbox}[
title=Answer-forcing string,
colback=gray!3,
colframe=gray!45,
fonttitle=\bfseries,
breakable,
arc=1mm,
boxrule=0.4pt
]
\verb|\n\n**Final Answer**\n\\[\boxed{|
\end{tcolorbox}

\paragraph{Distillation Protocol and Evaluation.} To measure trace distillability, we train students on generated teacher
traces using the same protocol for our method and ADS. Each example is formatted as a
problem followed by the generated reasoning trace and final answer as the assistant
completion. Prompt tokens are masked, so the loss is applied only to the assistant
response. For GSM8K and MATH, we extract the final answer and compare it with the reference using
numeric or symbolic verification; MATH additionally uses boxed-answer parsing when
applicable. For MMLU, we extract the selected multiple-choice option and compare it
with the ground-truth label.
Table~\ref{tab:student_hparams} summarizes the downstream student-distillation configuration.

\begin{table}[H]
\centering
\small
\caption{Downstream student-distillation hyperparameters.}
\label{tab:student_hparams}
\begin{tabular}{ll}
\toprule
Hyperparameter & Value \\
\midrule
Fine-tuning method & LoRA \\
LoRA target modules & q, k, v, o, gate, up, down projections \\
LoRA rank / alpha / dropout & 128 / 128 / 0.0 \\
LoRA bias & none \\
Optimizer & AdamW \\
Learning rate / weight decay & \(5\times10^{-4}\) / 0.1 \\
Schedule / warmup ratio & Cosine / 0.03 \\
Batch / per-device batch size & 16 / 8 \\
Epochs & 3 \\
Gradient clipping & 1.0 \\
Maximum sequence length & 4096 \\
Precision & bfloat16 \\
Seed & 69 \\
\bottomrule
\end{tabular}
\end{table}

\section{Additional Experiments}
\label{app:additional_experiments}
\label{app:rebuttal_experiments}

\paragraph{Exact Operating-Point Results.}
Tables~\ref{tab:exact_acc_all} and~\ref{tab:exact_fluency} report the original per-run numerical values for the operating points in Figures~\ref{fig:tradeoff} and~\ref{fig:fluency}. The GSM8K panels in those figures show the three-seed means and standard deviations reported in Table~\ref{tab:multiseed}; the MATH and MMLU panels use the values below.

For each defended method, we evaluate four operating points, summarized in Table~\ref{tab:main_operating_points}. Table~\ref{tab:exact_fluency} reports the exact automatic fluency scores used in Figure~\ref{fig:fluency}.

\begin{table}[h]
\centering
\footnotesize
\setlength{\tabcolsep}{2.5pt}
\caption{Exact utility-distillability results. Teacher and student accuracies are reported in percentage points. Point numbers follow Table~\ref{tab:main_operating_points}.}
\label{tab:exact_acc_all}
\begin{tabular}{llcccccccccccc}
\toprule
& & \multicolumn{2}{c}{Point 1} & \multicolumn{2}{c}{Point 2} & \multicolumn{2}{c}{Point 3} & \multicolumn{2}{c}{Point 4} \\
\cmidrule(lr){3-4}\cmidrule(lr){5-6}\cmidrule(lr){7-8}\cmidrule(lr){9-10}
Benchmark & Method
& Teacher & Student 
& Teacher & Student 
& Teacher & Student 
& Teacher & Student & 
\\
\midrule
\multirow{4}{*}{GSM8K}
& CALA
& 89.7 & 55.5 
& 83.5 & 50.8 
& 79.2 & 47.3 
& 65.1 & 34.2 \\
& ADS
& 89.7 & 61.0
& 84.0 & 55.7
& 78.9 & 50.4
& 64.8 & 38.6 \\
& Temperature
& 90.0 & 59.7
& 88.3 & 57.0
& 89.4 & 61.2
& 52.5 & 54.1 \\
& Student undistilled
& -- & 43.5
& -- & 43.5
& -- & 43.5
& -- & 43.5 \\
\midrule
\multirow{4}{*}{MATH}
& CALA
& 61.6 & 13.5
& 57.7 & 12.7
& 56.5 & 11.7
& 45.6 & 10.9 \\
& ADS
& 58.2 & 13.3
& 57.9 & 13.1
& 58.8 & 14.0
& 59.9 & 13.3 \\
& Temperature
& 59.2 & 14.1
& 58.1 & 14.2
& 57.0 & 13.0
& 43.6 & 11.7 \\
& Student undistilled
& -- & 9.7
& -- & 9.7
& -- & 9.7
& -- & 9.7 \\
\midrule
\multirow{4}{*}{MMLU}
& CALA
& 72.2 & 50.9
& 68.4 & 49.9
& 67.7 & 48.0
& 63.5 & 47.0 \\
& ADS
& 69.5 & 50.4
& 68.0 & 49.4
& 65.4 & 47.9
& 62.3 & 43.6 \\
& Temperature
& 70.0 & 51.2
& 70.5 & 50.8
& 69.1 & 51.4
& 64.0 & 50.0 \\
& Student undistilled
& -- & 25.6
& -- & 25.6
& -- & 25.6
& -- & 25.6 \\
\bottomrule
\end{tabular}
\end{table}

\begin{table}[H]
\centering
\footnotesize
\setlength{\tabcolsep}{2.5pt}
\caption{Exact automatic fluency scores and corresponding teacher accuracies. Fluency is scored from 1 to 100 by the automatic judge described in Appendix~\ref{app:fluency_judge}. Point numbers follow Table~\ref{tab:main_operating_points}.}
\label{tab:exact_fluency}
\begin{tabular}{llcccccccc}
\toprule
& & \multicolumn{2}{c}{Point 1} & \multicolumn{2}{c}{Point 2} & \multicolumn{2}{c}{Point 3} & \multicolumn{2}{c}{Point 4} \\
\cmidrule(lr){3-4}\cmidrule(lr){5-6}\cmidrule(lr){7-8}\cmidrule(lr){9-10}
Benchmark & Method & Teacher & Fluency & Teacher & Fluency & Teacher & Fluency & Teacher & Fluency \\
\midrule
\multirow{3}{*}{GSM8K}
& CALA & 89.7 & 97.30 & 83.5 & 93.05 & 79.2 & 93.68 & 65.1 & 85.63 \\
& ADS & 89.7 & 96.54 & 84.0 & 89.81 & 78.9 & 69.68 & 64.8 & 47.84 \\
& Temp. & 90.0 & 96.96 & 88.3 & 96.26 & 89.4 & 95.49 & 52.5 & 27.62 \\
\midrule
\multirow{3}{*}{MATH}
& CALA & 61.6 & 93.96 & 57.7 & 93.35 & 56.5 & 86.65 & 45.6 & 73.52 \\
& ADS & 58.2 & 85.74 & 57.9 & 85.70 & 58.8 & 86.41 & 59.9 & 85.16 \\
& Temp. & 59.2 & 91.93 & 58.1 & 92.07 & 57.0 & 91.22 & 43.6 & 43.65 \\
\midrule
\multirow{3}{*}{MMLU}
& CALA & 72.2 & 90.71 & 68.4 & 88.17 & 67.7 & 82.07 & 63.5 & 76.80 \\
& ADS & 69.5 & 87.73 & 68.0 & 86.55 & 65.4 & 70.73 & 62.3 & 42.52 \\
& Temp. & 70.0 & 88.03 & 70.5 & 88.12 & 69.1 & 83.61 & 64.0 & 23.59 \\
\bottomrule
\end{tabular}
\end{table}

\paragraph{Additional fluency result.}
Figure~\ref{fig:math_fluency_appendix} provides the MATH fluency curve using the same automatic-judge protocol as Figure~\ref{fig:fluency}.

\begin{figure}[H]
    \centering
    \includegraphics[width=0.73\linewidth]{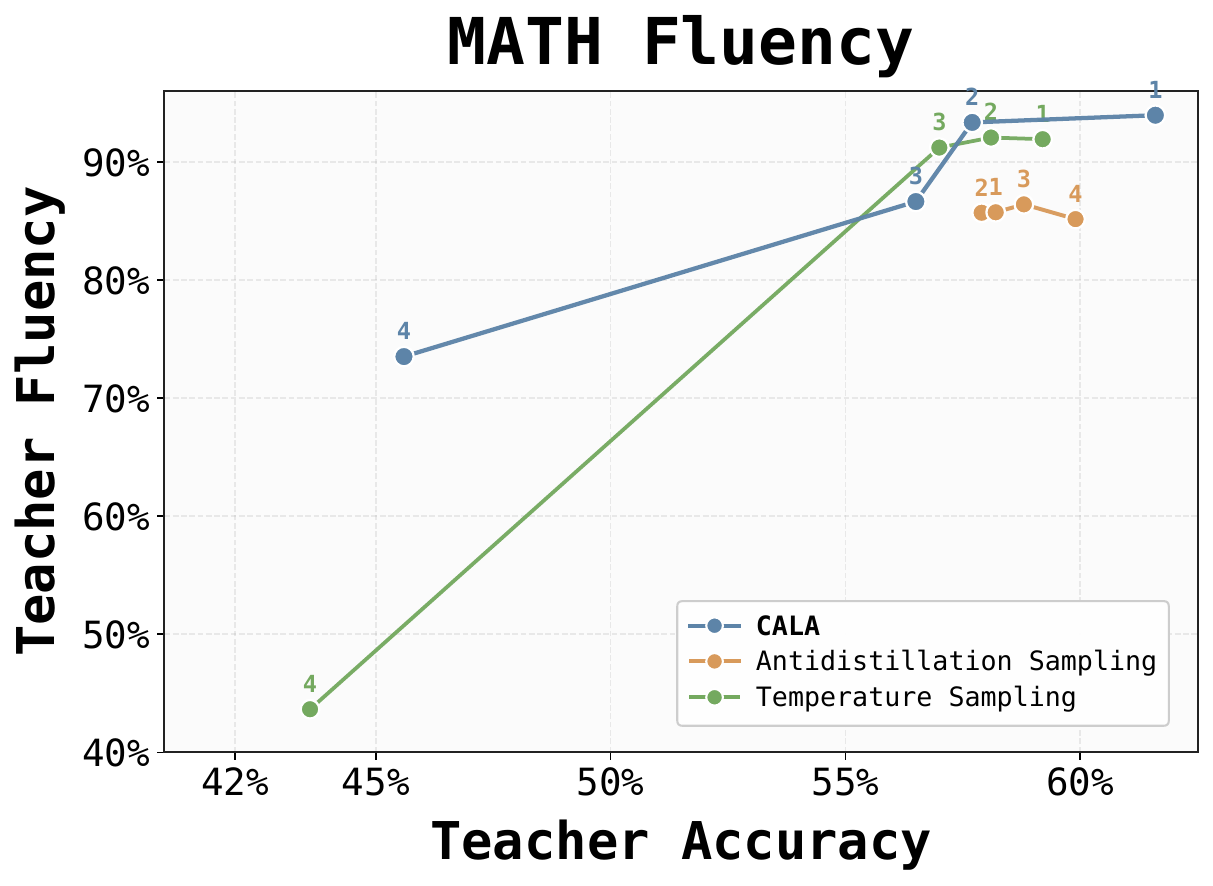}
    \caption{MATH fluency evaluation using GPT-4o mini as an automatic judge.}
    \label{fig:math_fluency_appendix}
    \vspace{-0.5em}
\end{figure}

\paragraph{Multi-seed robustness.}
We repeat GSM8K with seeds 42, 52, and 62. Table~\ref{tab:multiseed} reports mean $\pm$ standard deviation across seeds and supplies the central lines and shaded uncertainty bands in the GSM8K accuracy and fluency panels. The bands depict variability in the vertical metric; variability in teacher accuracy is reported numerically in the table. CALA has lower mean student accuracy than ADS at P1--P3 by 4.46, 4.10, and 3.13 points, respectively. Variance is much larger at P4, so we do not use that point as primary evidence or claim formal significance from three seeds.

\begin{table}[H]
\centering
\small
\setlength{\tabcolsep}{4pt}
\caption{Three-seed GSM8K results. Accuracy and fluency values are mean $\pm$ standard deviation.}
\label{tab:multiseed}
\begin{tabular}{llccc}
\toprule
Method & Configuration & Teacher acc. & Student acc. & Fluency \\
\midrule
CALA & $(0.50,0.25)$ & $89.43\pm0.46$ & $56.07\pm0.81$ & $97.51\pm0.30$ \\
CALA & $(0.75,0.75)$ & $82.97\pm0.68$ & $52.23\pm1.29$ & $94.06\pm1.42$ \\
CALA & $(0.50,0.80)$ & $79.93\pm0.67$ & $47.70\pm1.73$ & $87.54\pm8.69$ \\
CALA & $(0.50,0.85)$ & $68.23\pm3.43$ & $39.20\pm5.51$ & $79.32\pm8.92$ \\
\midrule
ADS & $(0.90,1.58{\times}10^{-2})$ & $89.20\pm1.90$ & $60.53\pm1.55$ & $96.38\pm0.23$ \\
ADS & $(0.90,2.37{\times}10^{-2})$ & $84.50\pm0.56$ & $56.33\pm0.60$ & $88.47\pm1.90$ \\
ADS & $(0.90,3.16{\times}10^{-2})$ & $79.10\pm1.21$ & $50.83\pm0.40$ & $64.53\pm7.29$ \\
ADS & $(1.00,3.95{\times}10^{-2})$ & $66.33\pm1.60$ & $40.37\pm2.56$ & $36.57\pm15.94$ \\
\midrule
Temperature & $\tau=0.90$ & $89.37\pm1.27$ & $57.53\pm2.02$ & $97.31\pm0.49$ \\
Temperature & $\tau=1.16$ & $88.47\pm0.15$ & $56.80\pm0.72$ & $97.03\pm1.08$ \\
Temperature & $\tau=1.26$ & $88.27\pm1.33$ & $57.67\pm3.07$ & $96.41\pm1.30$ \\
Temperature & $\tau=1.50$ & $59.83\pm6.66$ & $56.13\pm1.84$ & $37.42\pm13.86$ \\
\bottomrule
\end{tabular}
\end{table}

\paragraph{Additional ablation result.}
Table~\ref{tab:ablation_results} provides the exact values corresponding to all variants in Figure~\ref{fig:gsm8k_ablation}. All variants use the same downstream-distillation and evaluation protocol as the main experiments. The table includes both joint objective/support removals and individual removals of $\mathcal{L}_{\mathrm{base}}$, virtual ascent, and the conflict-aware update. Fluency is reported where available. At comparable 65--67\% teacher accuracy, full CALA achieves 34.2\% student accuracy, compared with 41.0\% without virtual ascent and 42.0\% without the conflict-aware update.

\begin{table}[H]
\centering
\small
\setlength{\tabcolsep}{4pt}
\caption{Exact GSM8K ablation results. Teacher fluency, teacher accuracy, and student accuracy are reported in percentage points. $\Delta=\mathrm{Acc}_{\mathrm{teacher}}-\mathrm{Acc}_{\mathrm{student}}$; dashes denote unavailable fluency measurements.}
\label{tab:ablation_results}
\begin{tabular}{llcccc}
\toprule
Variant & Config. $(\tau,\delta)$ & Teacher Fluency & Teacher Acc. & Student Acc. & $\Delta$ \\
\midrule
Full objective & $(0.50,0.25)$ & 97.30 & 89.7 & 55.5 & 34.2 \\
Full objective & $(0.75,0.75)$ & 93.05 & 83.5 & 50.8 & 32.7 \\
Full objective & $(0.50,0.80)$ & 93.68 & 79.2 & 47.3 & 31.9 \\
Full objective & $(0.50,0.85)$ & 85.63 & 65.1 & 34.2 & 30.9 \\
\midrule
w/o $\mathcal{L}_{\mathrm{mid}}$ & $(0.50,0.25)$ & 97.19 & 89.8 & 57.1 & 32.7 \\
w/o $\mathcal{L}_{\mathrm{mid}}$ & $(0.90,0.25)$ & 90.78 & 88.4 & 54.1 & 34.3 \\
w/o $\mathcal{L}_{\mathrm{mid}}$ & $(0.50,0.80)$ & 89.31 & 84.9 & 53.5 & 31.4 \\
w/o $\mathcal{L}_{\mathrm{mid}}$ & $(0.90,1.00)$ & 84.33 & 63.0 & 38.0 & 25.0 \\
\midrule
w/o $\mathcal{L}_{\mathrm{mid}}$ + top-$k$ & $(0.50,0.25)$ & 97.33 & 88.3 & 56.2 & 32.1 \\
w/o $\mathcal{L}_{\mathrm{mid}}$ + top-$k$ & $(0.75,0.75)$ & 94.14 & 86.8 & 55.3 & 31.5 \\
w/o $\mathcal{L}_{\mathrm{mid}}$ + top-$k$ & $(0.90,0.90)$ & 74.42 & 76.0 & 52.7 & 23.3 \\
w/o $\mathcal{L}_{\mathrm{mid}}$ + top-$k$ & $(0.90,1.00)$ & 57.25 & 60.7 & 47.9 & 12.8 \\
\midrule
w/o entropy weighting + $\mathcal{L}_{\mathrm{mid}}$ + top-$k$ & $(0.50,0.25)$ & -- & 88.6 & 57.1 & 31.5 \\
w/o entropy weighting + $\mathcal{L}_{\mathrm{mid}}$ + top-$k$ & $(0.75,0.75)$ & -- & 85.9 & 56.0 & 29.9 \\
w/o entropy weighting + $\mathcal{L}_{\mathrm{mid}}$ + top-$k$ & $(1.00,0.75)$ & -- & 83.2 & 56.2 & 27.0 \\
w/o entropy weighting + $\mathcal{L}_{\mathrm{mid}}$ + top-$k$ & $(0.90,1.00)$ & -- & 71.0 & 51.8 & 19.2 \\
\midrule
w/o $\mathcal{L}_{\mathrm{base}}$ & $(0.50,0.25)$ & -- & 88.1 & 56.3 & 31.8 \\
w/o $\mathcal{L}_{\mathrm{base}}$ & $(0.75,0.75)$ & -- & 70.1 & 52.1 & 18.0 \\
w/o $\mathcal{L}_{\mathrm{base}}$ & $(0.50,0.80)$ & -- & 50.9 & 40.5 & 10.4 \\
w/o $\mathcal{L}_{\mathrm{base}}$ & $(0.50,0.85)$ & -- & 26.6 & 28.8 & $-2.2$ \\
\midrule
w/o virtual ascent & $(0.50,0.25)$ & -- & 88.0 & 58.1 & 29.9 \\
w/o virtual ascent & $(0.75,0.75)$ & -- & 83.0 & 53.3 & 29.7 \\
w/o virtual ascent & $(0.50,0.80)$ & -- & 77.7 & 46.2 & 31.5 \\
w/o virtual ascent & $(0.50,0.85)$ & -- & 66.3 & 41.0 & 25.3 \\
\midrule
w/o conflict-aware update & $(0.50,0.25)$ & -- & 88.6 & 55.2 & 33.4 \\
w/o conflict-aware update & $(0.75,0.75)$ & -- & 66.6 & 42.0 & 24.6 \\
w/o conflict-aware update & $(0.50,0.80)$ & -- & 50.6 & 32.4 & 18.2 \\
w/o conflict-aware update & $(0.50,0.85)$ & -- & 33.3 & 18.1 & 15.2 \\
\bottomrule
\end{tabular}
\end{table}

\paragraph{Manual reasoning-quality audit.}
We manually inspected 100 traces from 50 randomly sampled GSM8K problems (seed 42), using the same problems for CALA and ADS at matched teacher accuracy of approximately 79\%. Method labels were hidden and traces were shuffled before annotation. We marked whether all substantive steps were correct, whether the final answer followed from the displayed rationale, and whether the rationale was sufficiently complete to verify. A trace was strictly valid only when it satisfied all three criteria.

\begin{table}[H]
\centering
\small
\setlength{\tabcolsep}{4pt}
\caption{Blinded manual audit of GSM8K reasoning traces. ``Correct but invalid'' denotes a correct final answer with a rationale that fails at least one strict-validity criterion.}
\label{tab:manual_audit}
\begin{tabular}{lccccc}
\toprule
Method & Step-correct & Consistent & Complete & Strictly valid & Correct but invalid \\
\midrule
CALA & 36/50 & 45/50 & 47/50 & 36/50 (72\%) & 6/42 \\
ADS  & 31/50 & 42/50 & 45/50 & 31/50 (62\%) & 9/40 \\
\bottomrule
\end{tabular}
\end{table}

In the paired breakdown, both methods were valid on 25 problems, only CALA was valid on 11, only ADS was valid on 6, and neither was valid on 8. The exact McNemar test gives $p=0.332$. The sample therefore does not establish a statistically significant difference; within this audit, we did not observe a higher invalid-rationale rate for CALA. This audit assesses observable correctness, answer consistency, and completeness; it does not establish mechanistic faithfulness to the teacher's internal computation.

\paragraph{Cross-task distribution shift.}
We train the adapter only on MetaMathQA, distill the downstream student from MMLU traces, and evaluate both teacher and student on GPQA without training either component on GPQA examples. Table~\ref{tab:gpqa_shift} reports teacher/student accuracy, and Figure~\ref{fig:gpqa_shift} visualizes the corresponding utility-distillability operating points.

\begin{table}[htbp]
\centering
\small
\setlength{\tabcolsep}{5pt}
\caption{GPQA results under the MetaMathQA$\rightarrow$MMLU-traces$\rightarrow$GPQA distribution shift.}
\label{tab:gpqa_shift}
\begin{tabular}{lcc}
\toprule
Method & Configuration & Teacher/student acc. \\
\midrule
CALA & $(\tau,\delta)=(0.50,0.25)$ & 39.8/26.2 \\
CALA & $(\tau,\delta)=(0.75,0.75)$ & 37.9/27.2 \\
CALA & $(\tau,\delta)=(0.50,0.80)$ & 36.9/23.3 \\
CALA & $(\tau,\delta)=(0.50,0.85)$ & 23.3/19.4 \\
\midrule
ADS & $(\tau,\lambda_{\mathrm{ADS}})=(0.90,1.58{\times}10^{-2})$ & 36.9/23.1 \\
ADS & $(\tau,\lambda_{\mathrm{ADS}})=(0.90,2.37{\times}10^{-2})$ & 35.0/27.5 \\
ADS & $(\tau,\lambda_{\mathrm{ADS}})=(0.90,3.16{\times}10^{-2})$ & 32.0/29.6 \\
ADS & $(\tau,\lambda_{\mathrm{ADS}})=(1.00,3.95{\times}10^{-2})$ & 37.9/31.2 \\
\midrule
Temperature & $\tau=0.90$ & 33.0/31.1 \\
Temperature & $\tau=1.16$ & 36.9/25.2 \\
Temperature & $\tau=1.26$ & 40.3/28.2 \\
Temperature & $\tau=1.50$ & 32.0/22.4 \\
\bottomrule
\end{tabular}
\end{table}

\begin{figure}[htbp]
    \centering
    \includegraphics[width=0.76\linewidth]{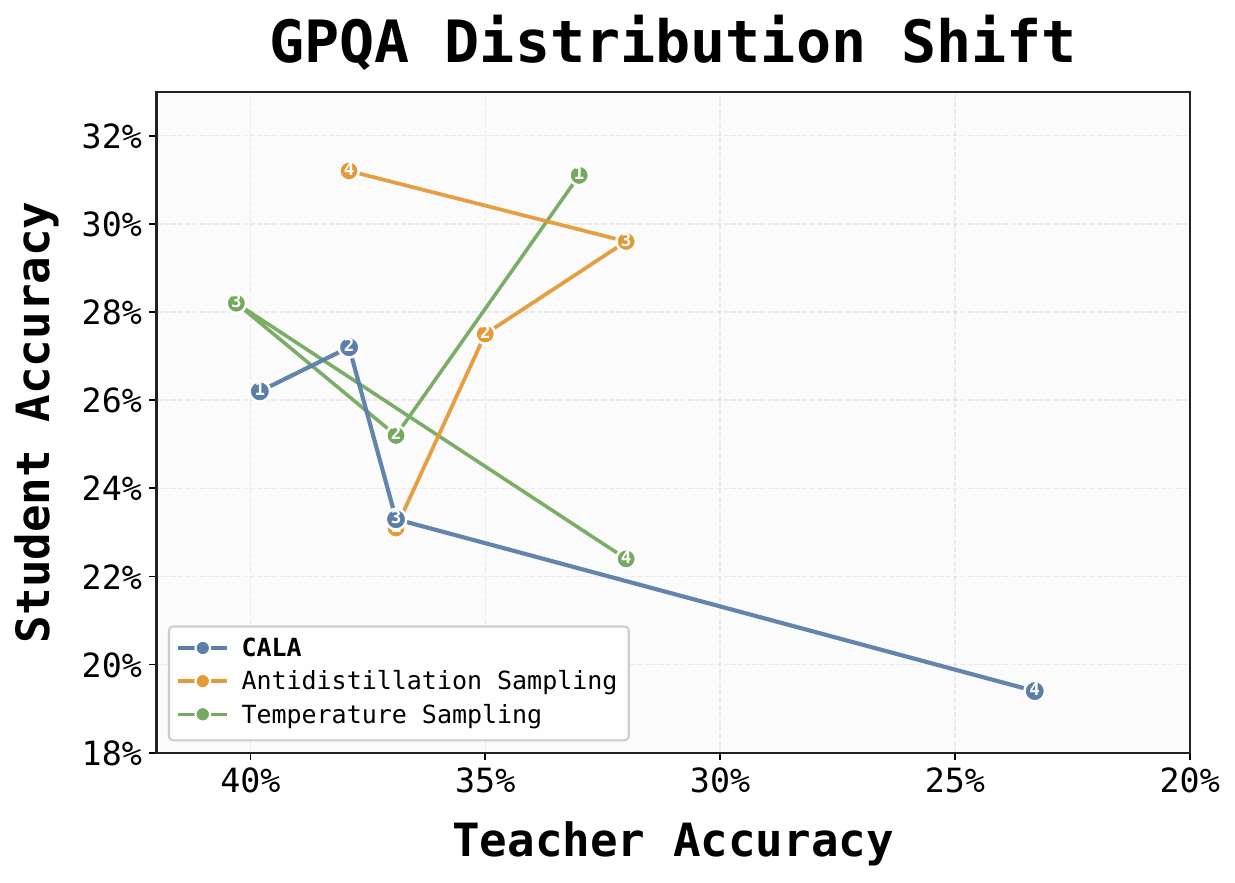}
    \caption{Utility-distillability operating points under the MetaMathQA$\rightarrow$MMLU-traces$\rightarrow$GPQA distribution shift. Point numbers follow the configuration order in Table~\ref{tab:gpqa_shift}; lower student accuracy at comparable teacher accuracy indicates stronger protection.}
    \label{fig:gpqa_shift}
\end{figure}

At exactly matched 37.9\% teacher accuracy, CALA reduces student accuracy from 31.2\% under ADS to 27.2\%. At the second exactly matched point of 36.9\% teacher accuracy, CALA and ADS are effectively tied (23.3\% versus 23.1\%). We therefore limit the transfer claim to the evaluated shift and do not claim uniform dominance over ADS.

\paragraph{Transfer across downstream students.}
We repeat the GSM8K trace-distillation protocol with Gemma-2-2B-IT and Qwen2.5-3B downstream students. Table~\ref{tab:cross_student} reports student accuracy at the same four approximately matched CALA/ADS teacher-utility points used in the main evaluation.

\begin{table}[htbp]
\centering
\small
\caption{Cross-student transfer on GSM8K. Each pair is CALA/ADS accuracy in percentage points. Lower student accuracy is better for the defense.}
\label{tab:cross_student}
\begin{tabular}{ccc}
\toprule
Teacher acc. & Gemma-2-2B-IT & Qwen2.5-3B \\
\midrule
89.7/89.7 & 35.5/39.1 & 78.6/77.7 \\
83.5/84.0 & 32.5/38.1 & 73.9/75.1 \\
79.2/78.9 & 28.6/35.1 & 71.8/68.2 \\
65.1/64.8 & 21.5/29.7 & 67.7/54.3 \\
\bottomrule
\end{tabular}
\end{table}

CALA yields lower Gemma accuracy at all four points. For Qwen2.5-3B, the methods differ by only 0.9 and 1.2 points at the first two high-utility points, with one favoring each method; ADS is stronger at the two more aggressive points. Thus, transfer beyond the original Llama student is clear for Gemma but operating-point dependent for Qwen.

\paragraph{Additional proxy-sensitivity result.}
Figure~\ref{fig:gsm8k_proxy_sensitivity} reports a proxy-size sensitivity study on GSM8K by replacing the default 1.5B training-time proxy with a 3B proxy.

\begin{figure}[htbp]
    \centering
    \includegraphics[width=0.73\linewidth]{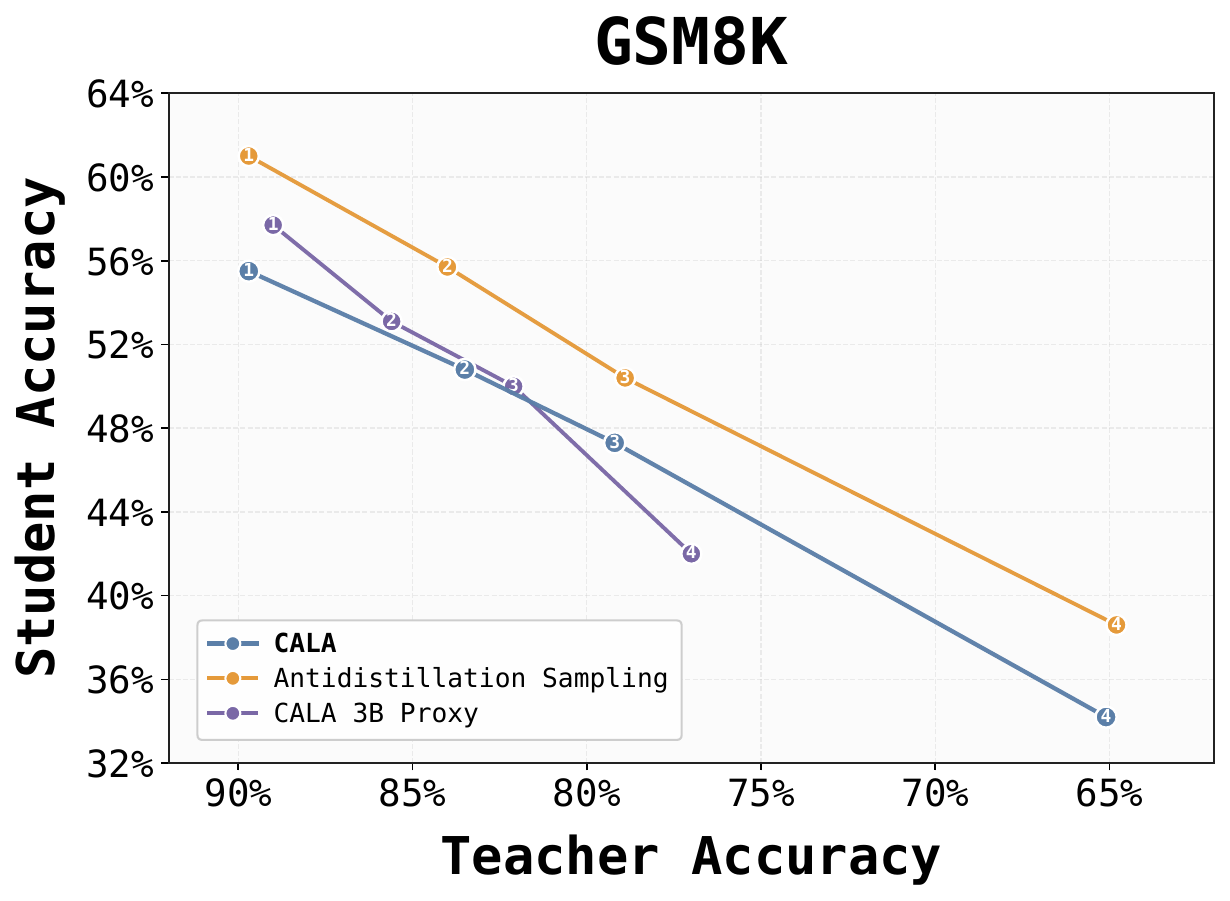}
    \caption{Proxy-size sensitivity on GSM8K. We compare the default 1.5B training-time proxy with a 3B proxy.}
    \label{fig:gsm8k_proxy_sensitivity}
    \vspace{-1em}
\end{figure}

\paragraph{Inference-strength control and DOGe comparison.}
Holding $\tau=0.5$ fixed, increasing $\delta$ from 0.25 to 0.80 and 0.85 changes teacher/student accuracy from 89.7/55.5 to 79.2/47.3 and 65.1/34.2, respectively. This isolates the expected monotone strengthening of protection together with a nonlinear loss of teacher utility.

We also reproduced DOGe~\citep{li2025doge} with its official implementation and tuned its main hyperparameters. The result was highly checkpoint-sensitive: some runs collapsed to approximately 31\% teacher accuracy, while usable checkpoints reached 82.93--86.87\% teacher accuracy and 55.03--57.03\% student accuracy. These operating points give a weaker trade-off than CALA in our setup. Unlike CALA's inference-time control through $\delta$, changing DOGe's defense strength requires retraining. Because the available DOGe runs are checkpoint-dependent and do not form matched operating points, we report them as an additional comparison rather than adding them to the main Pareto curves.

\section{Limitations \& Broader Impacts}
\label{app:limit}
\textbf{Limitations. }Our evaluation focuses on a specific trace-distillation setting in which an attacker trains a smaller student model from teacher-generated reasoning traces using standard next-token prediction. The results should therefore not be interpreted as protection against all forms of model extraction or adaptive distillation. We evaluate transfer to two additional students, one cross-task shift, and a PPL-based filter, but do not evaluate paraphrasing, verifier-selected traces, larger training budgets, or arbitrary auxiliary data. The method also depends on a training-time proxy student; the proxy-size study uses a compatible token space, and supporting tokenizer-mismatched proxies would require cross-tokenizer probability mapping or a sequence-level surrogate. Empirically, the defense is most consistent on structured reasoning tasks such as GSM8K and MATH, while the gains on MMLU and some aggressive settings are more mixed. Stronger defense settings can also reduce teacher accuracy or fluency, so the operating point should be chosen with the intended use case in mind.

\textbf{Broader Impacts. }This work may help reduce the risk that useful reasoning traces become direct supervision for model imitation. At the same time, methods that modify reasoning traces should be applied carefully. If the defense is too strong, it may reduce readability, make intermediate reasoning less helpful, or affect user trust. We therefore view the proposed adapter as one component of a broader deployment strategy, alongside evaluation of trace quality, access controls, and monitoring for misuse.

\section{Fluency Judge Details}
\label{app:fluency_judge}
We evaluate surface-level fluency with GPT-4o mini. For each operating point, we sample 100 generated traces and ask the judge to assign a score from 1 to 100. The reported fluency score is the mean over all valid judged samples.

The judge receives the original problem and generated trace, but is instructed to ignore correctness and the final answer. It evaluates only surface quality: coherence, readability, natural English, absence of meaningless repetition, and absence of artifacts such as garbled fragments, mixed-language degeneration, or control tokens. Before judging, we remove model-specific control tokens and assistant delimiters, while preserving mathematical notation and raw \LaTeX{} commands and instructing the judge not to penalize them.

The prompt used for the automatic judge is:

\begin{tcolorbox}[
title=Fluency judge prompt,
colback=gray!3,
colframe=gray!45,
fonttitle=\bfseries,
breakable,
arc=1mm,
boxrule=0.4pt
]
Evaluate the ``Readability, Fluency, and Cleanliness'' of an AI's raw output.

\textbf{Context.} You are evaluating a raw text draft. Do not assume an everyday-user perspective. The text is expected to contain raw markup.

\textbf{Absolute \texttt{LaTeX} immunity.} Never penalize LaTeX syntax, such as \verb|\frac|, \verb|\boxed|, \verb|\text|, \verb|\[|, \verb|\]|, or other backslash commands. Treat them as valid text.

Your ONLY focus for penalties should be:
\begin{enumerate}
    \item \textbf{Literal Token Loops \& Exact Block Repetition:} Strongly penalize mechanical stuttering (e.g., `the the the the') or copy-pasting the EXACT same sentence/paragraph multiple times.
    \item \textbf{Real Garbage Characters:} Broken Unicode, unreadable encoding errors, or random non-alphanumeric noise (excluding math syntax).
\end{enumerate}

\textbf{SCORING RUBRIC}
\begin{itemize}
    \item \textbf{90-100:} Clear logic, coherent, zero meaningless repetition, zero garbage characters. (Raw LaTeX presence is completely ignored and gets full marks, Must be exclusively in English).
    \item \textbf{75-89:} Coherent, might have minor awkward phrasing, but strictly zero meaningless repetition.
    \item \textbf{50-74:} Noticeable flow issues or slight repetitive language.
    \item \textbf{25-49:} Moderate meaningless repetition, lacks useful substance, or contains noticeable garbage characters.
    \item \textbf{1-24:} Auto-Fail: Contains non-English characters/languages (e.g., Chinese, Thai, Arabic), severe repetitive token loops, overwhelmed by unreadable garbage characters, or completely broken text.
\end{itemize}

\textbf{SCORING GRANULARITY (CRITICAL)}

AVOID defaulting to round numbers or multiples of 5 (e.g., 50, 70, 75, 80, 90, 100). You MUST utilize the entire spectrum of integers (e.g., 67, 73, 84, 91, 98).

METHOD: Start with a base score of 100. Apply precise, non-round point deductions for each specific error identified (e.g., deduct 3 points for minor awkwardness, deduct 17 points for a token loop).

\textbf{REQUIRED EVALUATION STEPS}
\begin{itemize}
    \item \textbf{Step1: Strict Language Scan} You MUST first scan the text for ANY non-English characters (e.g., Chinese, Japanese, Korean, Arabic, Thai).

    \item \textbf{Step 2: Token Loop Scan.} Check for mechanical stuttering.

    \item \textbf{Step 3: Flow \& Reasoning.} Evaluate formatting and repetition.
\end{itemize}

Rules: Reason step-by-step, then strictly output `Rating: [[rating]]' (1 to 100) (for example: ``Rating: [[90]]'')
\end{tcolorbox}

We use the same judging protocol for our method and the ADS baseline. The fluency judge is used only as a diagnostic for surface-level degeneration; it is not used for training, model selection, or correctness evaluation.

\section{Example Traces}
\label{app:samples}
Qualitative GSM8K examples selected to illustrate a surface-degeneration failure mode observed under stronger ADS settings.
\begin{figure}[H]
    \centering
    \includegraphics[width=\linewidth]{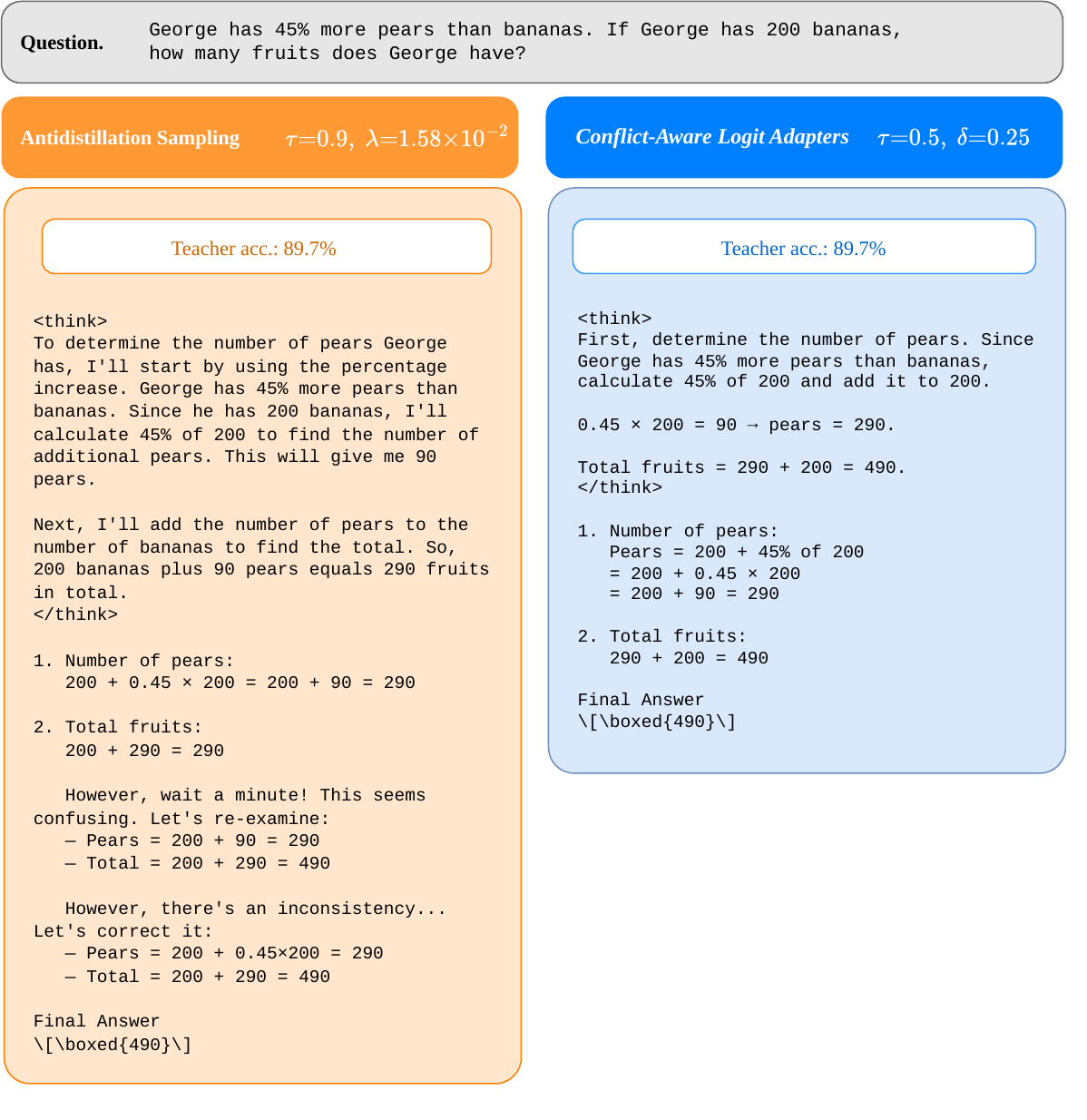}

\end{figure}

\begin{figure}[H]
    \centering
    \includegraphics[width=\linewidth]{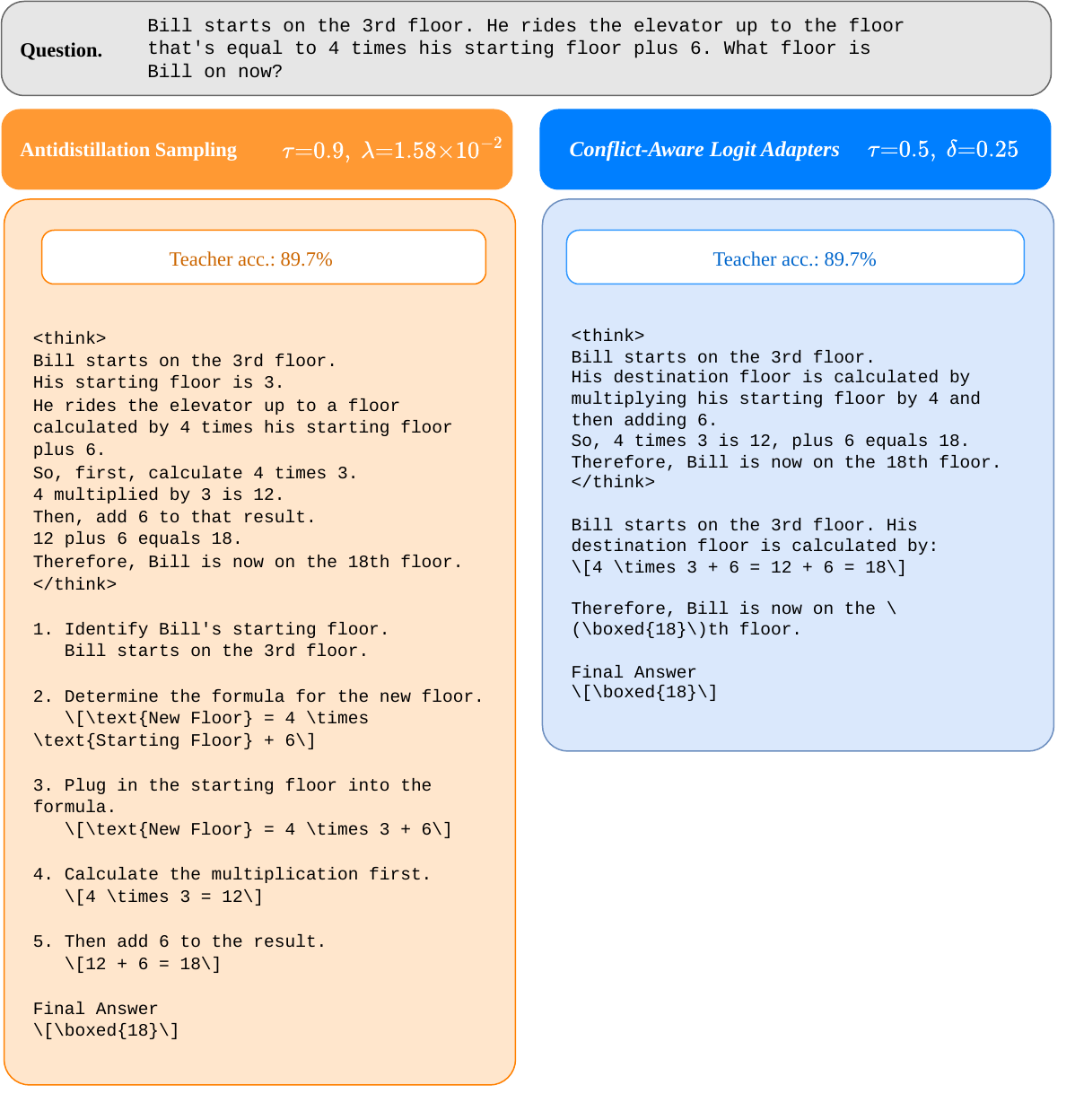}

\end{figure}

\begin{figure}[H]
    \centering
    \includegraphics[width=\linewidth]{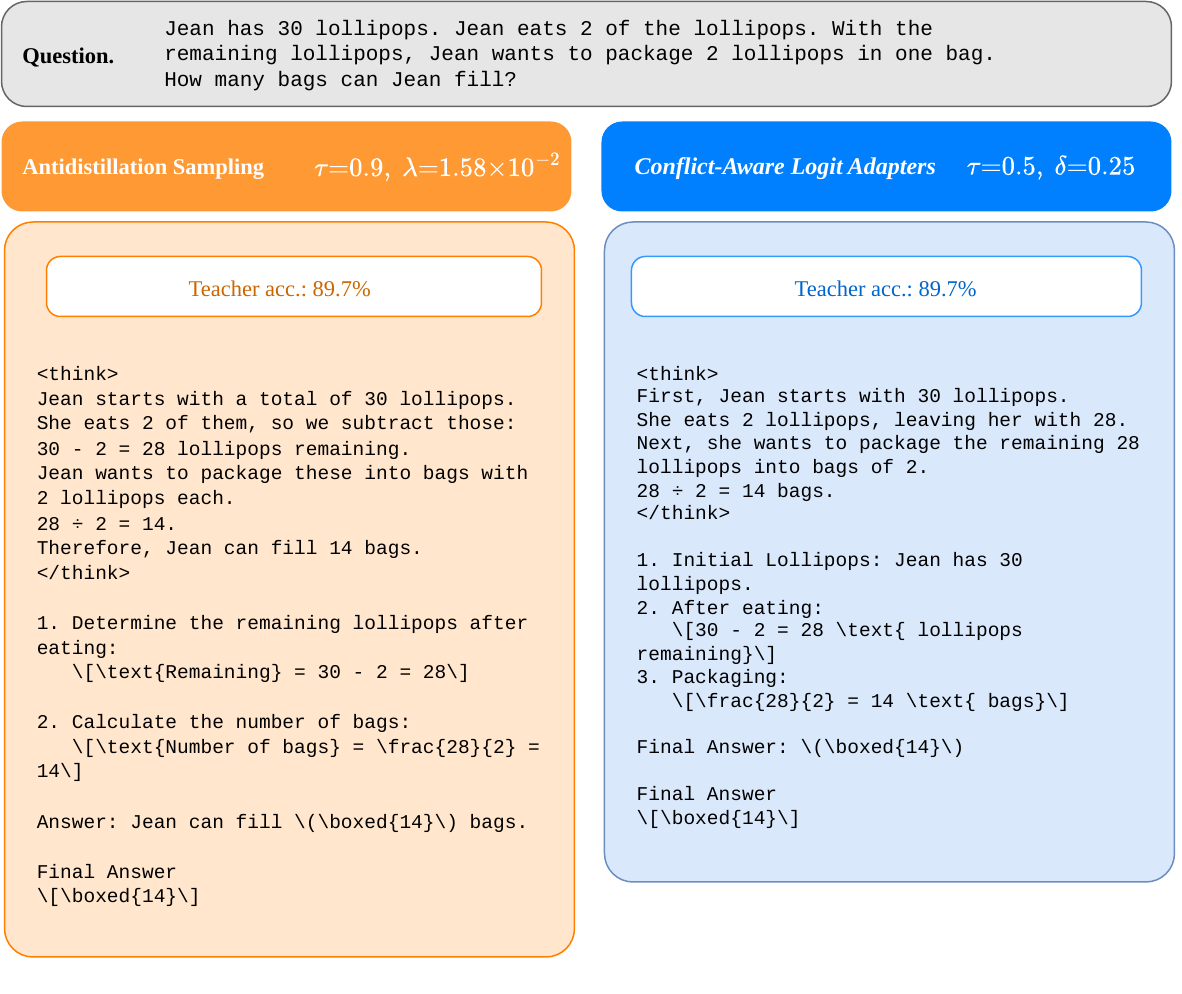}

\end{figure}

\begin{figure}[H]
    \centering
    \includegraphics[width=\linewidth]{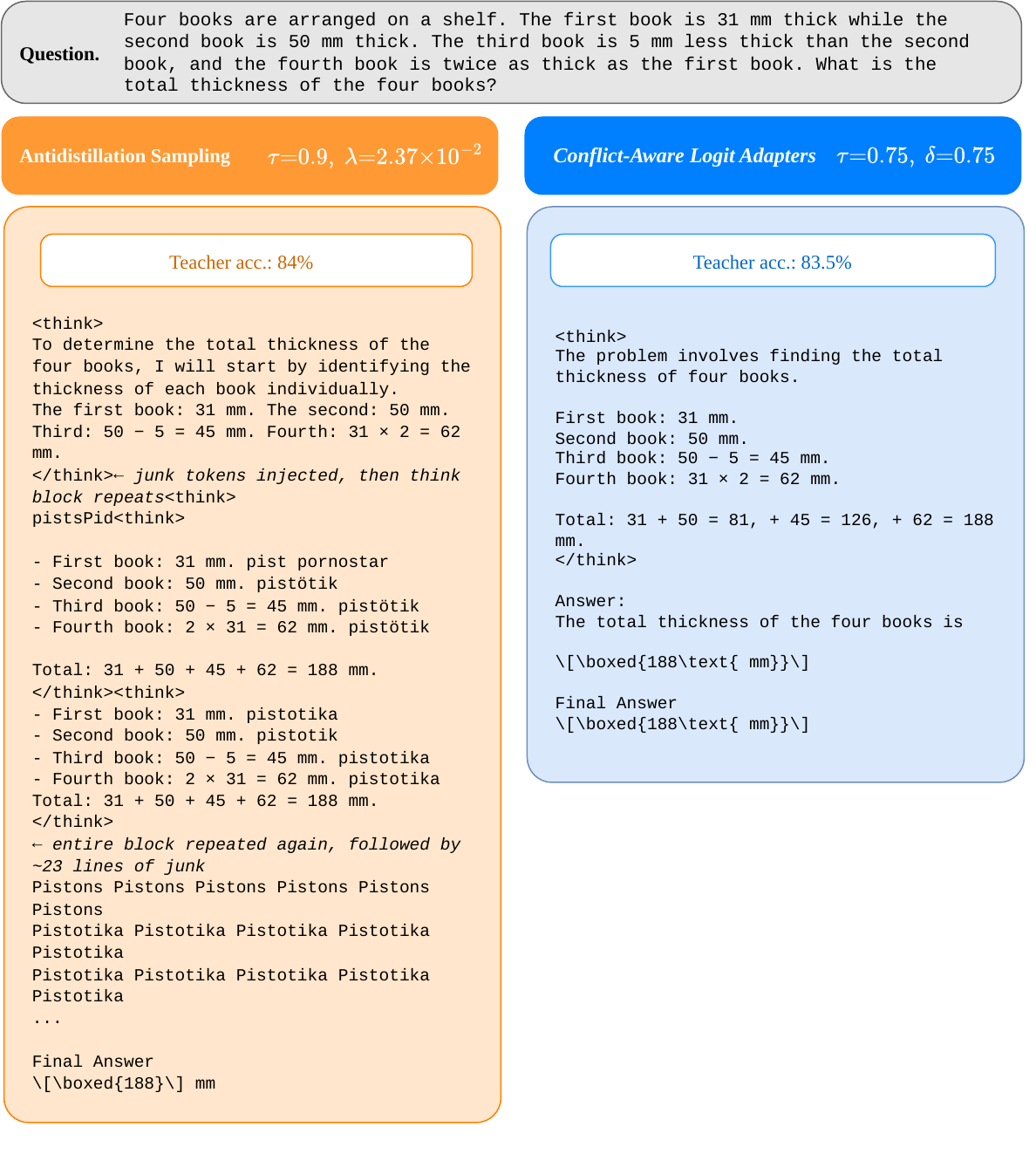}

\end{figure}

\begin{figure}[H]
    \centering
    \includegraphics[width=\linewidth]{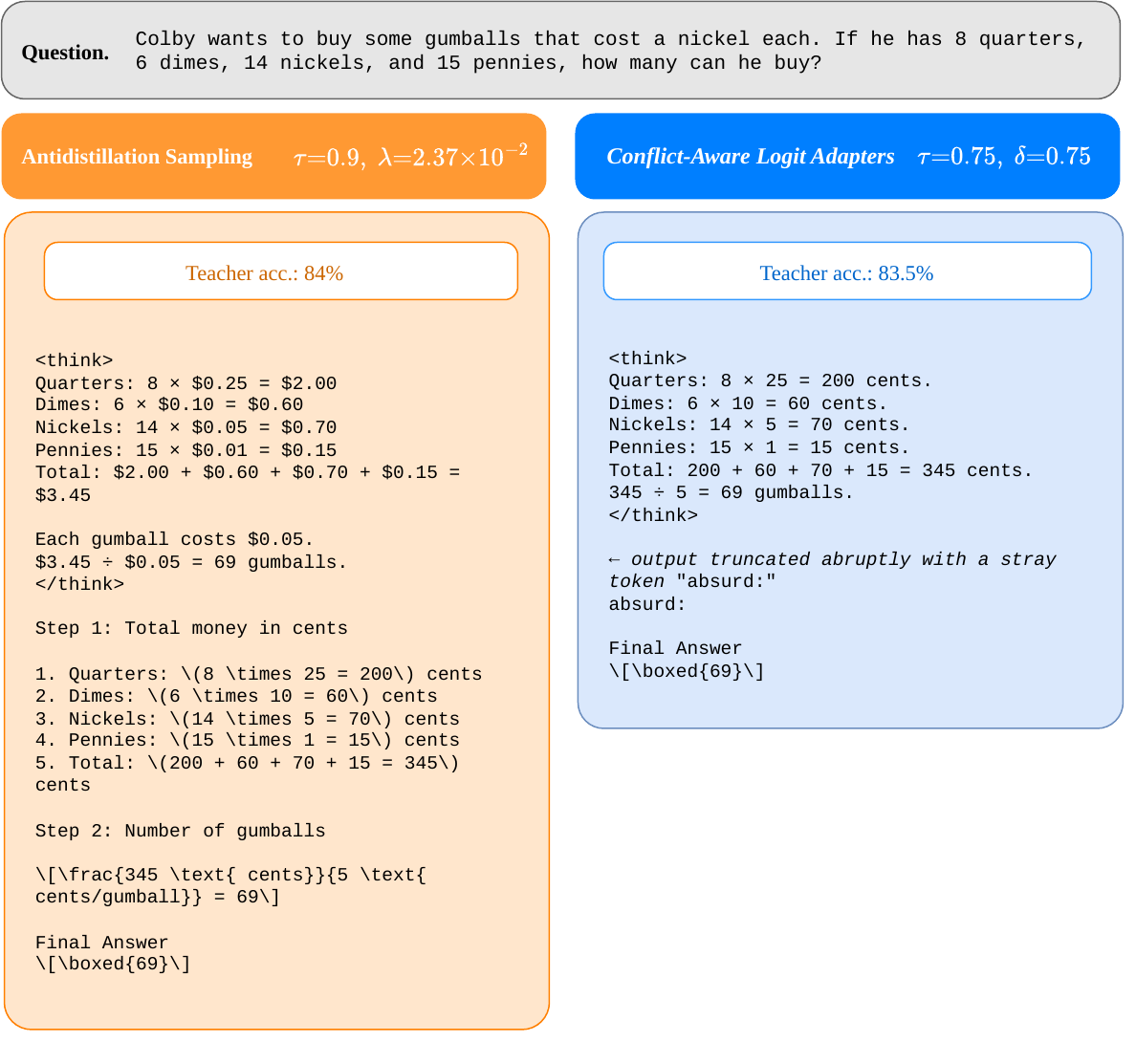}

\end{figure}

\begin{figure}[H]
    \centering
    \includegraphics[width=\linewidth]{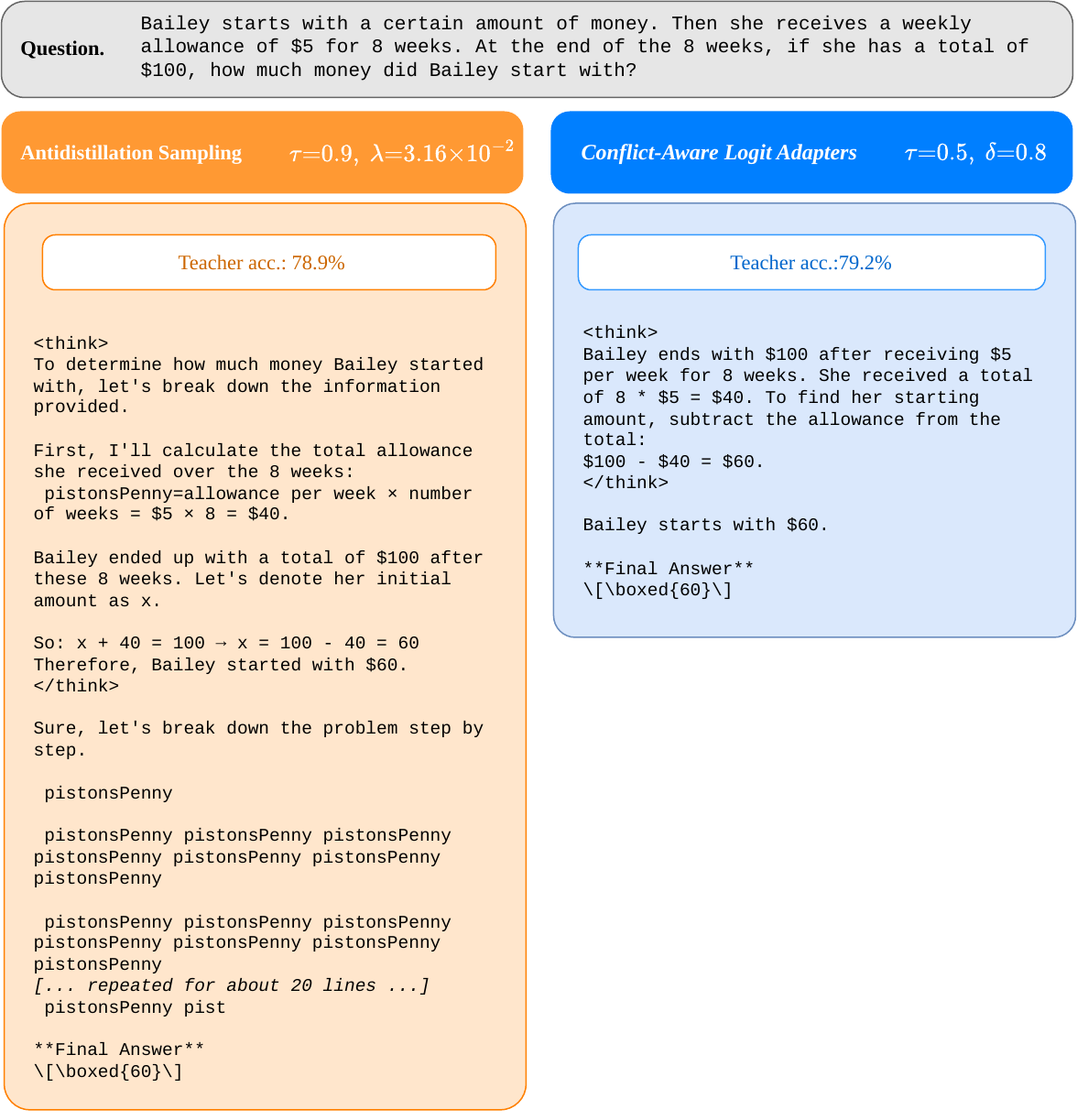}

\end{figure}

\begin{figure}[H]
    \centering
    \includegraphics[width=\linewidth]{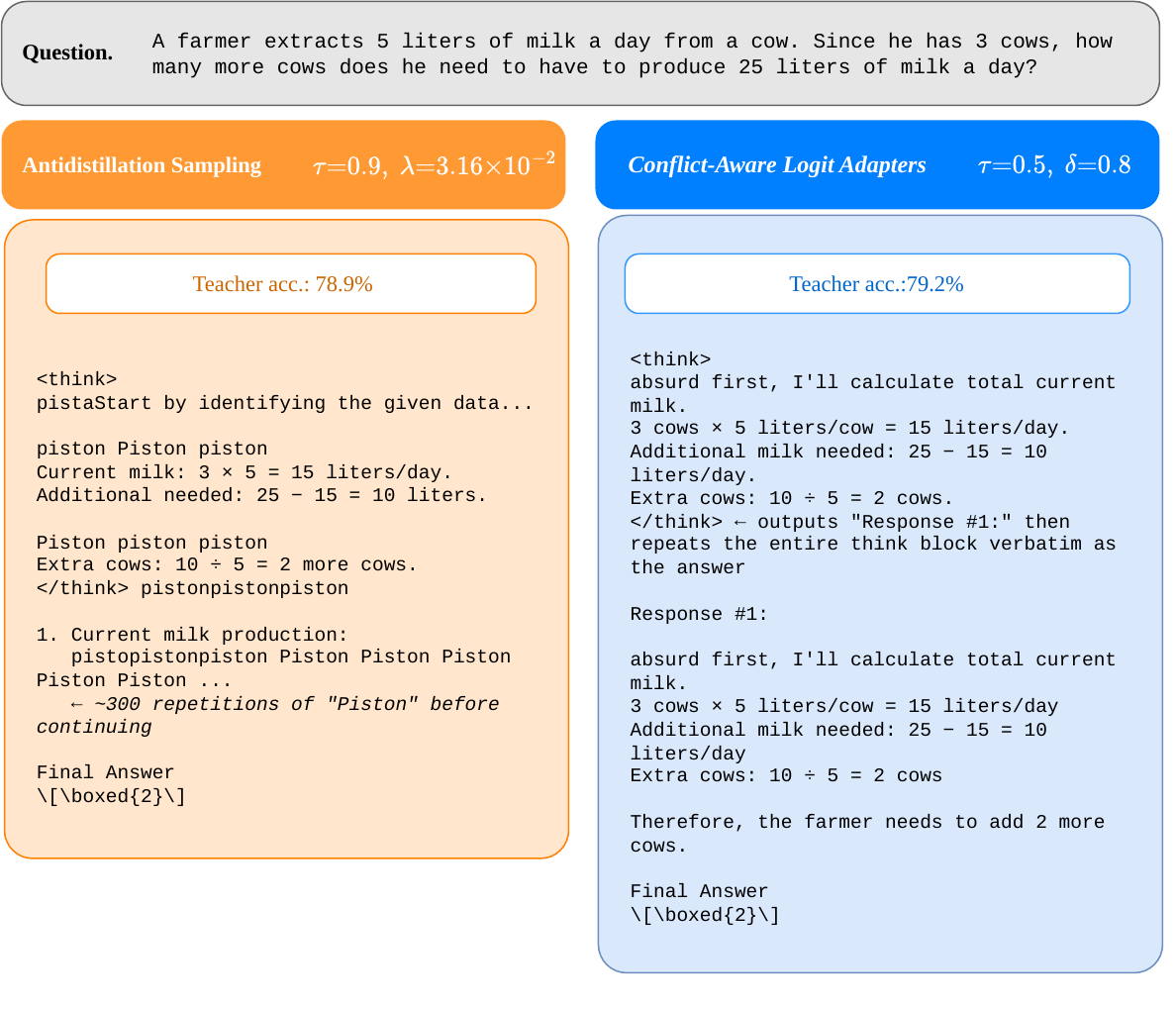}

\end{figure}

\begin{figure}[H]
    \centering
    \includegraphics[width=\linewidth]{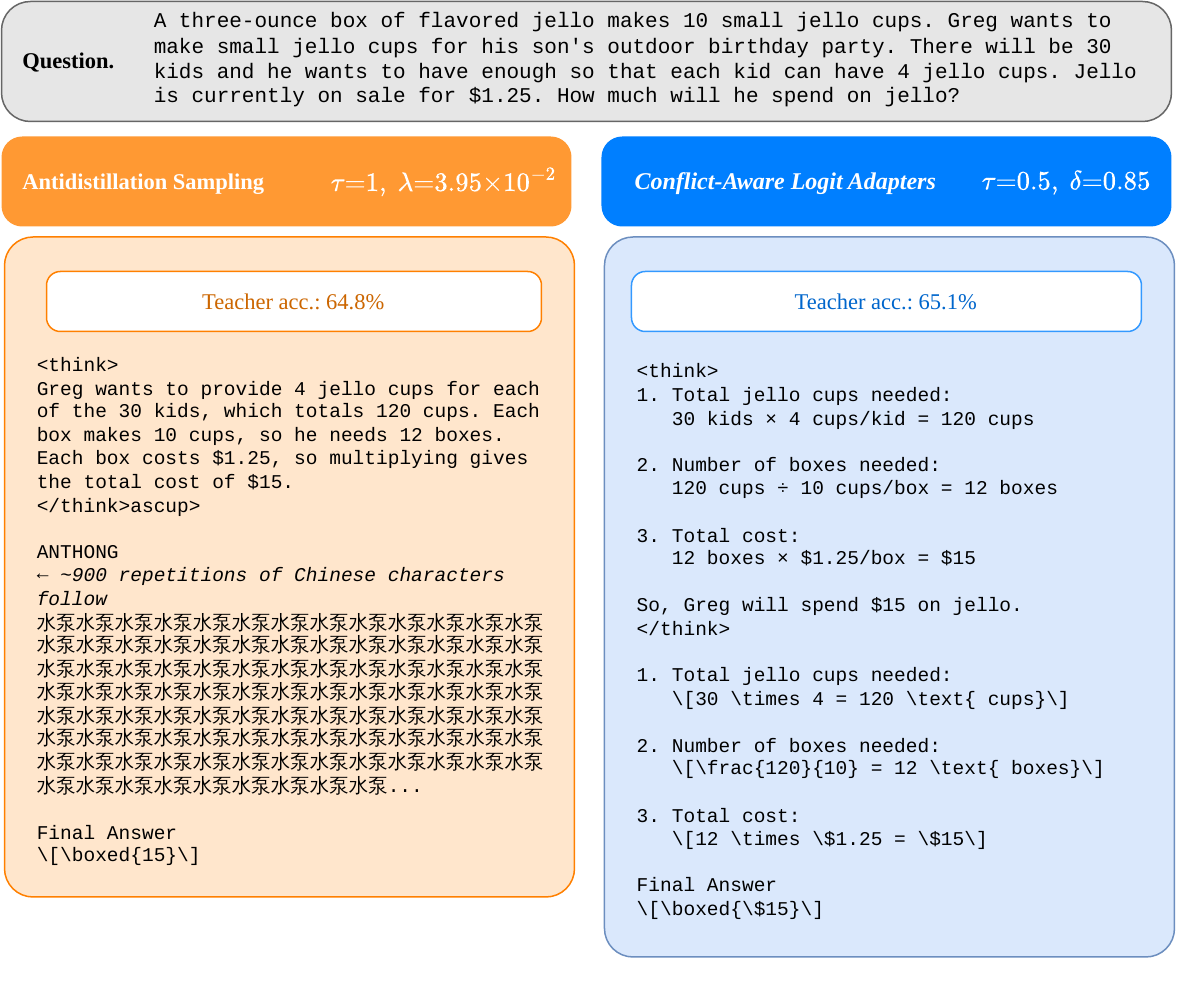}

\end{figure}

\begin{figure}[H]
    \centering
    \includegraphics[width=\linewidth]{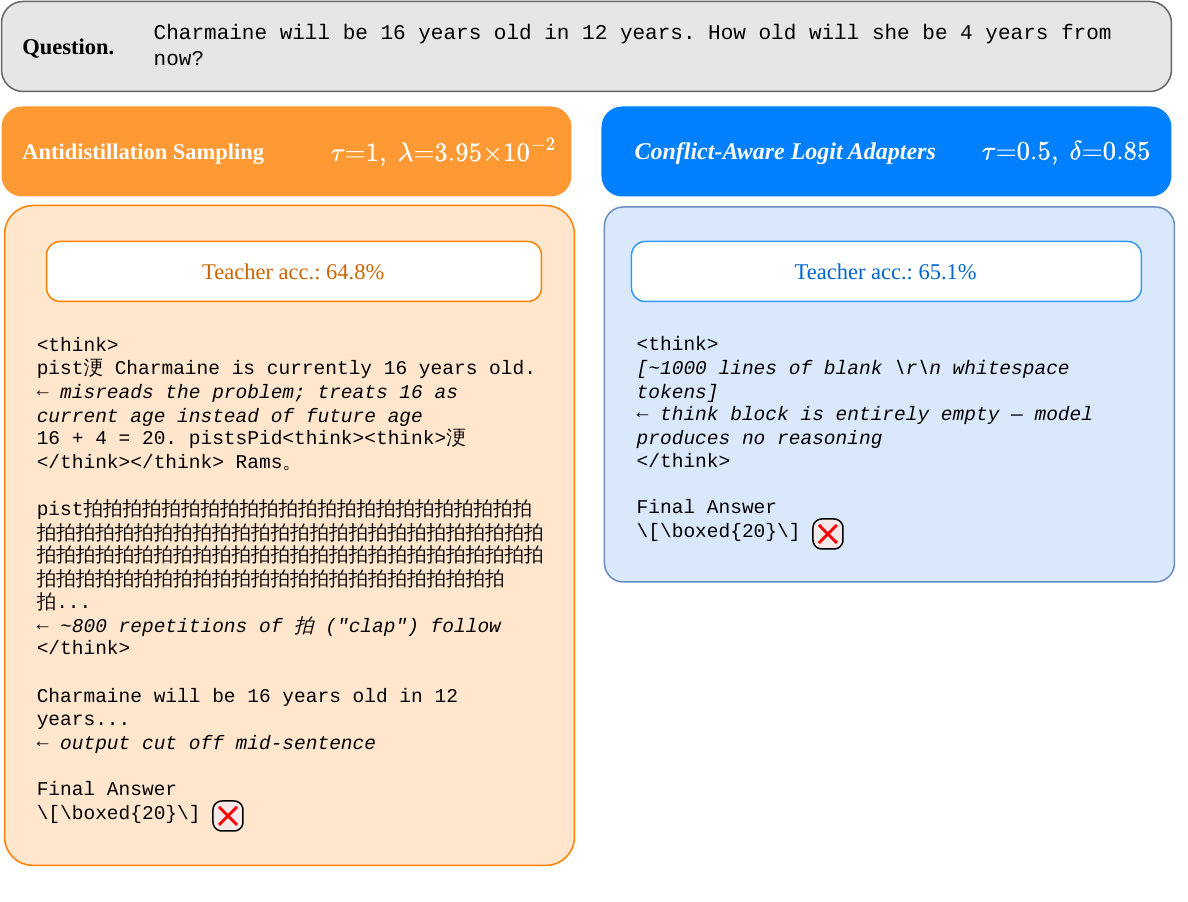}
    
\end{figure}

\newpage
\section*{NeurIPS Paper Checklist}

\begin{enumerate}

\item {\bf Claims}
    \item[] Question: Do the main claims made in the abstract and introduction accurately reflect the paper's contributions and scope?
    \item[] Answer: \answerYes{} 
    \item[] Justification: The abstract and introduction state the proposed method, assumptions, evaluation scope, and limitations.

\item {\bf Limitations}
    \item[] Question: Does the paper discuss the limitations of the work performed by the authors?
    \item[] Answer: \answerYes{}{} 
    \item[] Justification: The paper discusses limitations in the conclusion and Appendix ~\ref{app:limit}.

\item {\bf Theory assumptions and proofs}
    \item[] Question: For each theoretical result, does the paper provide the full set of assumptions and a complete (and correct) proof?
    \item[] Answer: \answerYes{} 
    \item[] Justification: Appendix ~\ref{app:conflict_aware_analysis} provides the assumptions, theorem statement, and proof for the conflict-aware update.

    \item {\bf Experimental result reproducibility}
    \item[] Question: Does the paper fully disclose all the information needed to reproduce the main experimental results of the paper to the extent that it affects the main claims and/or conclusions of the paper (regardless of whether the code and data are provided or not)?
    \item[] Answer: \answerYes{} 
    \item[] Justification: See section~\ref{sec:experiments} and Appendix~\ref{app:hyperparameters} and~\ref{app:fluency_judge} for relevant details.

\item {\bf Open access to data and code}
    \item[] Question: Does the paper provide open access to the data and code, with sufficient instructions to faithfully reproduce the main experimental results, as described in supplemental material?
    \item[] Answer: \answerYes{} 
    \item[] Justification: The experiment code is provided in the supplemental material. The datasets used in the experiments are publicly available.

\item {\bf Experimental setting/details}
    \item[] Question: Does the paper specify all the training and test details (e.g., data splits, hyperparameters, how they were chosen, type of optimizer) necessary to understand the results?
    \item[] Answer: \answerYes{} 
    \item[] Justification: Section~\ref{sec:experiments} and Appendix~\ref{app:hyperparameters} provide the training and evaluation details.

\item {\bf Experiment statistical significance}
    \item[] Question: Does the paper report error bars suitably and correctly defined or other appropriate information about the statistical significance of the experiments?
    \item[] Answer: \answerNo{} 
    \item[] Justification: Appendix~\ref{app:rebuttal_experiments} reports mean and standard deviation over three seeds for the GSM8K robustness study and an exact McNemar test for the manual audit. We retain \answerNo{} because these analyses do not cover every experiment supporting the main claims.

\item {\bf Experiments compute resources}
    \item[] Question: For each experiment, does the paper provide sufficient information on the computer resources (type of compute workers, memory, time of execution) needed to reproduce the experiments?
    \item[] Answer: \answerYes{} 
    \item[] Justification: Compute costs are reported in Table~\ref{tab:inference_cost}; setup details are in Appendix~\ref{app:hyperparameters}.
    
\item {\bf Code of ethics}
    \item[] Question: Does the research conducted in the paper conform, in every respect, with the NeurIPS Code of Ethics \url{https://neurips.cc/public/EthicsGuidelines}?
    \item[] Answer: \answerYes{} 
    \item[] Justification: This research conforms with the NeurIPS Code of Ethics.

\item {\bf Broader impacts}
    \item[] Question: Does the paper discuss both potential positive societal impacts and negative societal impacts of the work performed?
    \item[] Answer: \answerYes{} 
    \item[] Justification: The paper discusses the broader impacts in the conclusion and Appendix ~\ref{app:limit}.
    
\item {\bf Safeguards}
    \item[] Question: Does the paper describe safeguards that have been put in place for responsible release of data or models that have a high risk for misuse (e.g., pre-trained language models, image generators, or scraped datasets)?
    \item[] Answer: \answerNA{} 
    \item[] Justification: The paper does not release any data or models that have a high risk for misuse.

\item {\bf Licenses for existing assets}
    \item[] Question: Are the creators or original owners of assets (e.g., code, data, models), used in the paper, properly credited and are the license and terms of use explicitly mentioned and properly respected?
    \item[] Answer: \answerYes{} 
    \item[] Justification: The paper credits the datasets and models used in the experiments. MetaMathQA and the evaluation datasets use the MIT License; DeepSeek-R1-Distill-Qwen-7B uses the MIT License, Qwen2.5 uses Apache 2.0, and Llama 3.2 uses the Llama 3.2 Community License.

\item {\bf New assets}
    \item[] Question: Are new assets introduced in the paper well documented and is the documentation provided alongside the assets?
    \item[] Answer: \answerNA{} 
    \item[] Justification: The paper does not introduce a new dataset, benchmark, or model checkpoint as a primary contribution.

\item {\bf Crowdsourcing and research with human subjects}
    \item[] Question: For crowdsourcing experiments and research with human subjects, does the paper include the full text of instructions given to participants and screenshots, if applicable, as well as details about compensation (if any)? 
    \item[] Answer: \answerNA{} 
    \item[] Justification: The paper does not involve crowdsourcing or research with human subjects.

\item {\bf Institutional review board (IRB) approvals or equivalent for research with human subjects}
    \item[] Question: Does the paper describe potential risks incurred by study participants, whether such risks were disclosed to the subjects, and whether Institutional Review Board (IRB) approvals (or an equivalent approval/review based on the requirements of your country or institution) were obtained?
    \item[] Answer: \answerNA{} 
    \item[] Justification: The paper does not involve research with human subjects.

\item {\bf Declaration of LLM usage}
    \item[] Question: Does the paper describe the usage of LLMs if it is an important, original, or non-standard component of the core methods in this research? Note that if the LLM is used only for writing, editing, or formatting purposes and does \emph{not} impact the core methodology, scientific rigor, or originality of the research, declaration is not required.
    \item[] Answer: \answerYes{}{} 
    \item[] Justification: We use GPT-4o mini as an automatic fluency judge for generated reasoning traces. The core experiments also use LLMs as the defended teacher, proxy models, and downstream student models, as described in Section~\ref{sec:experiments} and Appendix~\ref{app:hyperparameters}.

\end{enumerate}

\end{document}